\documentclass{article}

\PassOptionsToPackage{authoryear}{natbib}
\usepackage[preprint]{neurips_2023}

\usepackage[utf8]{inputenc} % allow utf-8 input
\usepackage[T1]{fontenc}    % use 8-bit T1 fonts
\usepackage{hyperref}       % hyperlinks
\usepackage{url}            % simple URL typesetting
\usepackage{booktabs}       % professional-quality tables
\usepackage{amsfonts}       % blackboard math symbols
\usepackage{nicefrac}       % compact symbols for 1/2, etc.
\usepackage{microtype}      % microtypography
\usepackage{xcolor}         % colors

\usepackage{algorithm}
\usepackage{algorithmic}
\usepackage{amsmath}
\usepackage{amssymb}
\usepackage{mathtools}
\usepackage{amsthm}
\usepackage{thm-restate}

\theoremstyle{plain}
\newtheorem{theorem}{Theorem}[section]
\newtheorem{example}{Example}
\newtheorem{proposition}[theorem]{Proposition}
\newtheorem{lemma}[theorem]{Lemma}

\theoremstyle{definition}

\theoremstyle{remark}

\renewcommand\thmcontinues[1]{\emph{cont.}}

\usepackage[textsize=tiny]{todonotes}
\usepackage{bm}
\usepackage{thm-restate}

\usepackage{xspace}
\newcommand{\rido}{RIDO\@\xspace}

\newcommand{\E}{\mathbb{E}}

\newcommand{\Var}{\mathbb{V}\mathrm{ar}}

\newcommand{\Cov}{\mathbb{C}\mathrm{ov}}
\newcommand{\Empvar}{\widehat{\mathbb{V}}\mathrm{ar}}
\newcommand{\Empcov}{\widehat{\mathbb{C}}\mathrm{ov}}

\title{Truncating Trajectories in Monte Carlo Policy Evaluation: an Adaptive Approach}

\author{%
  Riccardo Poiani \\ %\thanks{Use footnote for providing further information
  DEIB, Politecnico di Milano\\
  \texttt{riccardo.poiani@polimi.it} \\
  \And
  Nicole Nobili \\
  DEIB, Politecnico di Milano \\
  \texttt{nicole.nobili@mail.polimi.it} \\
  \AND
  Alberto Maria Metelli \\
  DEIB, Politecnico di Milano  \\
  \texttt{albertomaria.metelli@polimi.it} \\
  \And
  Marcello Restelli \\
  DEIB, Politecnico di Milano \\
  \texttt{marcello.restelli@polimi.it} \\
}

\begin{document}

\maketitle

\begin{abstract}
{\color{black}
Policy evaluation via Monte Carlo (MC) simulation is at the core of many MC Reinforcement Learning (RL) algorithms (e.g., policy gradient methods). In this context, the designer of the learning system specifies an interaction budget that the agent usually spends by collecting trajectories of \emph{fixed length} within a simulator. However, is this data collection strategy the best option? To answer this question, in this paper, we propose as a quality index a surrogate of the mean squared error of a return estimator that uses trajectories of different lengths, i.e., \emph{truncated}. Specifically, this surrogate shows the sub-optimality of the fixed-length trajectory schedule. Furthermore, it suggests that adaptive data collection strategies that spend the available budget sequentially can allocate a larger portion of transitions in timesteps in which more accurate sampling is required to reduce the error of the final estimate. Building on these findings, we present an \emph{adaptive} algorithm called \textbf{R}obust and \textbf{I}terative \textbf{D}ata collection strategy \textbf{O}ptimization (\rido). The main intuition behind \rido is to split the available interaction budget into mini-batches. At each round, the agent determines the most convenient schedule of trajectories that minimizes an empirical and robust version of the surrogate of the estimator's error. After discussing the theoretical properties of our method, we conclude by assessing its performance across multiple domains. Our results show that \rido can adapt its trajectory schedule toward timesteps where more sampling is required to increase the quality of the final estimation.
}
\end{abstract}

%\am{Ho uniformato il lessico in questo modo:\\
%- FIXED: solo per indicare traiettorie di lunghezza fissa\\
%- NON-ADAPTIVE: schedule calcolato prima dell'interazione\\
%- ADAPTIVE: quello che facciamo
%}
%The abstract paragraph should be indented \nicefrac{1}{2}~inch (3~picas) on
%both the left- and right-hand margins. Use 10~point type, with a vertical
%spacing (leading) of 11~points.  The word \textbf{Abstract} must be centered,
%bold, and in point size 12. Two line spaces precede the abstract. The abstract
%must be limited to one paragraph.

\section{Introduction}\label{sec:intro}
In Reinforcement Learning \citep[RL,][]{sutton2018reinforcement}, an agent acts in an unknown, or partially known, environment to maximize/estimate the infinite expected discounted sum of an external reward signal, i.e., the expected return. Monte Carlo evaluation \citep[MC,][]{owen2013monte} is at the core of many successful RL algorithms. Whenever a simulator with reset possibility is available to the learning systems designer, a large family of approaches \citep{williams1992simple,baxter2001infinite,schulman2015trust,schulman2017proximal,cobbe2021phasic} that can be used to solve the RL problem relies on MC simulations for estimating performance or gradient estimates on the task to be solved. In this scenario, since the goal is to estimate the expected infinite sum of rewards, the designer usually specifies a sufficiently large estimation horizon $T$, along with a transition budget $\Lambda = QT$, so that the agent interacts with the simulator, via MC simulation, collecting a batch of $Q$ episodes of length $T$.\footnote{While another large class of RL algorithms is based on Temporal Difference \citep[TD,][]{sutton2018reinforcement} learning, which do not require the finite horizon nor the reset possibility, Monte Carlo simulation approaches continue to be extensively adopted. Indeed, unlike TD methods, they can be applied effortlessly to non-Markovian environments, which is a common occurrence in real-world problems.} In this sense, the agent spends its available budget $\Lambda$ \emph{uniformly along the estimation horizon}.

In the context of MC policy evaluation, where the goal lies in estimating the performance of a given policy via MC simulations, \citet{poiani2022truncating} have recently shown that, given the discounted nature of the RL objective, this uniform-in-the-horizon budget allocation strategy may not be the best option. The core intuition behind their work is that, since rewards are exponentially discounted through time, early interactions weigh exponentially more than late ones, and, consequently, a larger portion of the available budget $\Lambda$ should be dedicated to estimating the initial rewards. To theoretically validate this point, the authors designed a \emph{non-adaptive} budget allocation strategy which, by exploiting the reset possibility of the simulator, leads to the collection of trajectories of different lengths, i.e., \emph{truncated}. They show that this approach provably minimizes H\"oeffding-like confidence intervals \citep{boucheron2003concentration} around the empirical estimates of the expected return. Remarkably, this implies a robustness w.r.t. the uniform strategy that holds for any pair of environment and policy to be evaluated, thus, clearly establishing the theoretical benefits of the proposed method. 

%Nevertheless, it has to be noticed that although minimizing confidence intervals around the expected return estimator comes with desirable theoretical guarantees (e.g., PAC-bound improvements \citep{even2002pac}), the resulting schedule of trajectories is inherently designed from a \emph{worst-case} perspective. Furthermore, similarly to the usual uniform-in-the-horizon scheme, the approach of~\cite{poiani2022truncating} delivers a schedule computed \emph{before} the interaction with the environment (being determined by the discount factor) which fails to adapt to the peculiarities of the problem at hand, and, ultimately, might not produce a low error estimate. For the sake of clarity, we illustrate the sub-optimality of pre-determined schedule of trajectories with the following extreme examples.

Nevertheless, it has to be noticed that although minimizing confidence intervals around the expected return estimator comes with desirable theoretical guarantees (e.g., PAC-bound improvements \citep{even2002pac}), the resulting schedule of trajectories is computed \emph{before} the interaction with the environment (being determined by the discount factor). Consequently, as the usual uniform-in-the-horizon scheme, it fails to adapt to the peculiarities of the problem at hand, and, ultimately, might not produce a low error estimate. For the sake of clarity, we illustrate this sub-optimality of pre-determined schedule of trajectories with the following extreme examples.

\begin{example}\label{exe:1}
Consider an environment where the reward is gathered only at the end of the horizon $T$ (e.g., a goal-based). In this scenario, any strategy that truncates trajectories is intuitively sub-optimal, and we expect that an intelligent agent will spend all its budget according to the uniform schedule. 
\end{example}
\begin{example}\label{exe:2}
Conversely, consider a problem where the reward is different from $0$ in the first interaction step only (e.g., in the case of a highly sub-optimal policy that immediately reaches the ``zero reward region'' of an environment); the uniform schedule wastes a significant portion of its budget collecting samples without variability, and, to reduce the estimation error, we would like the agent to spend all of its interaction budget estimating the reward of the first action. 
\end{example}

Abstracting away from the previous examples, we realize that the main issue of existing approaches arises from the fact that determining a schedule of trajectories \emph{before} interacting with the environment does not allow the agent to adapt it to the environment peculiarities, allocating more samples where this is required to obtain a high-quality estimate. For this reason, in this work, we focus on designing \emph{adaptive} data collection strategies that aim directly at minimizing the \emph{error} of the final estimate. Our main intuition lies in splitting the available budget $\Lambda$ into mini-batches and adapting \emph{online} the data collection strategy of the agent based on the previously collected information.

\textbf{Original Contributions and Outline}~~{\color{black}After introducing the necessary notation and backgrounds (Section \ref{sec:prelim}), we consider the problem of maximizing the estimation quality of a policy expected return estimator using trajectories of different lengths collected via MC simulation with a finite budget $\Lambda$ of transitions (Section \ref{sec:opt-dcs}). More specifically, we adopt the return estimator proposed in \cite{poiani2022truncating}, and we analyze it to propose a surrogate of the error of the final estimation for every schedule of trajectories. As we shall see, this surrogate will show (i) the sub-optimality of schedule of trajectories that are determined prior to interacting with the environment, and (ii) it suggests that algorithms that spend the available budget $\Lambda$ iteratively might be able to dynamically allocate their budget to minimize the error of the final estimate. Building on these findings, in Section \ref{sec:method}, we present a {novel} algorithm, \textbf{R}obust and \textbf{I}terative \textbf{D}ata collection strategy \textbf{O}ptimization (\rido), which splits its available budget $\Lambda$ into mini-batches of interactions that are allocated sequentially to minimize an empirical and robust version of the surrogate of the estimator's error. Furthermore, we perform a statistical analysis on the behavior of \rido, and we derive theoretical guarantees expressed as upper bounds on the surrogate of the policy return error. To conclude, in Section \ref{sec:exp}, we conduct an experimental comparison between \rido and non-adaptive schedules. As we verify, our method achieves the most competitive performance across different domains, discount factor values, and budget, thus clearly highlighting the benefits of adaptive strategies over pre-determined ones.}

\section{Backgrounds and Notation}\label{sec:prelim}
This section provides the notation and necessary backgrounds used in the rest of this document.

\textbf{Markov Decision Processes}~~A discrete-time Markov Decision Process \citep[MDP,][]{puterman2014markov} is defined as a tuple $\mathcal{M} \coloneqq \left( \mathcal{S}, \mathcal{A}, R, P, \gamma, \nu \right)$, where $\mathcal{S}$ is the set of states, $\mathcal{A}$ is the set of actions, $R: \mathcal{S} \times \mathcal{A} \rightarrow [0,1]$ is the reward function the specifies the reward $R(s,a)$ received by the agent upon taking action $a$ in state $s$, $P: \mathcal{S} \times \mathcal{A} \rightarrow \Delta(\mathcal{S})$\footnote{Given a set $\mathcal{X}$, we denote with $\Delta(\mathcal{X})$ the set of probability distributions over $\mathcal{X}$.} is the transition kernel that specifies the probability distribution over the next states $P(\cdot | s, a)$, when taking action $a$ in state $s$, $\gamma \in (0, 1)$ is the discount factor, and $\nu \in \Delta(\mathcal{S})$ is the distribution over initial states. The agent's behavior is modeled by a policy $\pi: \mathcal{S} \rightarrow \Delta(\mathcal{A})$, which for each state $s$, prescribes a distribution over actions $\pi(\cdot|s)$. A trajectory $\bm{\tau}_h$ of length $h$ is a sequence of states and actions $\left(s_0, a_0, s_1, \dots s_{h-1}, a_{h-1}, s_h \right)$ observed by following $\pi$ for $h$ steps, where $s_0 \sim \nu$, $a_t \sim \pi(\cdot|s_t)$, and $s_{t+1} \sim P(\cdot|s_t,a_t)$ for $t < h$. The return of a trajectory is defined as $G(\bm{\tau}_h) = \sum_{t=0}^{h-1} \gamma^t R_t$, where $R_t$ is shortcut for $R(s_t, a_t)$. 
The agent that is following  policy $\pi$ is evaluated according to expected cumulative discounted sum of rewards over an estimation horizon $T$,\footnote{As common in Monte-Carlo simulation \citep[see e.g.,][]{papini2022smoothing} we approximate the infinite horizon MDP model with a finite estimation horizon $T$. Indeed, if $T$ is sufficiently large, i.e., $T = \mathcal{O}\left(\frac{1}{1-\gamma} \log \frac{1}{\epsilon} \right)$, the expected return computed with horizon $T$ is $\epsilon$ close to the infinite-horizon one \citep{kakade2003sample}.}, namely $J(\pi) = \E_\pi \left[ \sum_{t=0}^{T-1} \gamma^t R_t \right]$, where the expectation is taken w.r.t. the stochasticity of the policy, the transition kernel, and the initial state distribution.

%In the infinite horizon MDP model, the performance of an agent that is following $\pi$ is evaluated according to the expected cumulative discounted sum of rewards, namely $\E_\pi \left[ \sum_{t=0}^{+\infty} \gamma^t R(a_t, s_t) \right]$, where the expectation is taken w.r.t. the stochasticity of the policy, the transition kernel, and the initial state distribution. However, when using Monte-Carlo based simulation methods, an estimation horizon $T$ is usually considered 

\textbf{Data Collection Strategy}~~\citet{poiani2022truncating} formalized the concept of Data Collection Strategy (DCS) to model how the agent collects data within an environment. More specifically, given an interaction budget $\Lambda \in \mathbb{N}$ such that $\Lambda \mod T = 0$, a DCS is defined as a $T$-dimensional vector $\bm{m} \coloneqq (m_1, \dots, m_T)$ where $m_h \in \mathbb{N}$ and $\sum_{h=1}^{T} m_h h = \Lambda$. Each element $m_h$ specifies the number of trajectories of length $h$ that the agent collects in the environment while following a policy $\pi$. Given a DCS $\bm{m}$, it is possible to compute the total number of steps $\bm{n} \coloneqq \left(n_0, \dots, n_{T-1} \right)$ that will be gathered by the agent at any step $t$; more specifically, the following relationship holds: $n_{T-1} = m_{T}$, and $n_t = n_{t+1} + m_{t+1}$ for $t < T-1$. For this reason, in the rest of the paper we will adopt the most convenient symbol depending on the context. For any DCS $\bm{m}$ such that $m_T \ge 1$ holds, it is possible to build the following estimator of $J(\pi)$:
\begin{align}\label{eq:ttmc-estimator}
\hat{J}_{\bm{m}}(\pi) = \sum_{h=1}^{T} \sum_{i=1}^{m_h} \sum_{t=0}^{h-1} \frac{\gamma^t}{n_t} R_t^{(i)}.
\end{align}
The two external summations in Equation \eqref{eq:ttmc-estimator} sum over the collected trajectories of different lengths a rescaled empirical trajectory return, where the reward at step $t$ is divided by the number of samples collected at step $t$. When the budget $\Lambda$ is spent uniformly, i.e., $\bm{m}= \left(0, \dots, 0, \frac{\Lambda}{T} \right)$, Equation \eqref{eq:ttmc-estimator} reduces to the usual Monte Carlo estimator of $J(\pi)$, namely $\frac{T}{\Lambda} \sum_{i=1}^{\Lambda/T} \sum_{t=0}^{T-1} \gamma^t {R_t^{(i)}}$. {\color{black} Whenever $\bm{m}$ is computed deterministically (e.g., prior to interacting with the environment), it is possible to prove that Estimator \eqref{eq:ttmc-estimator} is an unbiased estimate of $J(\pi)$ (see Theorem B.1 in \cite{poiani2022multi})}.

\textbf{Robust Data Collection Strategy Optimization}~~Leveraging the estimator of Equation \eqref{eq:hoeffding-gen}, \citet{poiani2022truncating} investigated alternatives to the usual uniform-in-the-horizon DCS from the worst-case perspective of confidence intervals \citep{boucheron2003concentration}. More specifically, given $\bm{m}$ such that $m_T \ge 1$, the estimator of Equation \eqref{eq:ttmc-estimator} enjoys the following generalization of the H\"oeffding confidence intervals holding with probability at least $1-\delta$:
\begin{align}\label{eq:hoeffding-gen}
|J(\pi) - \hat{J}_{\bm{m}}(\pi)| \le \sqrt{\frac{1}{2} \log\left(\frac{2}{\delta} \right) \sum_{t=0}^{T-1} \frac{d_t}{n_t}},
\end{align}
where $d_t = \frac{\gamma^t \left(\gamma^t + \gamma^{t+1} -2\gamma^T \right)}{1-\gamma}$ controls the relative importance of samples gathered at step $t$. \citet{poiani2022truncating} designed a closed-form DCS that provably minimizes the bound of Equation \eqref{eq:hoeffding-gen}. Since $d_t$ is a decreasing function of time whose decay speed is governed by the discount factor $\gamma$, the aforementioned DCS gives priority to the collection of experience at earlier time steps, i.e., it truncates the trajectories. Note that the smaller $\gamma$, the higher the number of samples reserved for earlier time steps. We refer the reader for Theorem 3.3 and Theorem B.10 of their work for the exact expressions of the resulting robust DCS. However, we remark that the resulting schedule is non-adaptive (i.e., it is computed before the interaction with the environment takes place) and its shape depends exclusively on $\Lambda$, $\gamma$, and $T$.\footnote{{\color{black} As a consequence, combining Equation \eqref{eq:ttmc-estimator} with the DCS that minimizes Equation \eqref{eq:hoeffding-gen} provides an unbiased estimate of $J(\pi)$.}}

\section{Toward Adaptive Data Collection Strategies}\label{sec:opt-dcs}
In this section, we lay down the theoretical groundings behind optimizing data collection strategies that directly aim at minimizing the final estimation error. 
{\color{black}  We stick to methods that adopt the estimator of Equation \eqref{eq:ttmc-estimator}. It has to be remarked that Equation \eqref{eq:ttmc-estimator} has been mainly analyzed in \cite{poiani2022truncating} for deterministic DCSs. Since the goal of our work is to propose adaptive DCS algorithms, we analyze the property of this estimator when combined with stochastic algorithms that sequentially spend the budget $\Lambda$ according to information collected in the past. To highlight the fact that the resulting DCS is now a random variable, we make use of capital letters, i.e., we write $\bm{M}$ (or, equivalently, $\bm{N}$). In these cases, unfortunately, Equation \eqref{eq:ttmc-estimator} is no longer an unbiased estimates of $J(\pi)$. This is due the fact that rewards $R_t$ observed by the agent are now correlated with $\bm{M}$ and $\bm{N}$. As the next theorem shows, however, it is still possible to prove that Equation \eqref{eq:ttmc-estimator} is consistent whenever
\begin{align}\label{eq:consistency}
\lim_{\Lambda \rightarrow +\infty} \Lambda^T \exp(-C N_T) = 0
\end{align}
holds almost surely for any constant $C > 0$. Specifically, we prove the following result (proof in Appendix~\ref{app:proofs}).
\begin{restatable}{theorem}{consistency}\label{theo:consist}
Fix $\epsilon > 0$ and consider an online algorithm such that Equation \eqref{eq:consistency} holds almost surely. Then, we have that:
\begin{align*}
\lim_{\Lambda \rightarrow +\infty} \mathbb{P}(|\hat{J}_{\bm{M}}(\pi) - J(\pi)| > \epsilon) = 0.
\end{align*}
\end{restatable}
Equation \eqref{eq:consistency} enforces a determinstic behavior on the online algorithm used to output $\bm{M}$ in the limiting case where $\Lambda \rightarrow +\infty$. This is satisfied, for example, when an amount of samples proportional (up to costant multiplicative factor) to the budget $\Lambda$ will be deterministically allocated to the last interaction step. This condition will be verified by the online algorithm that we will present in Section~\ref{sec:method}.

Given this preliminary consideration, we now derive an upper-bound of the error of the Mean Squared Error (MSE) of the estimator of Equation~\eqref{eq:ttmc-estimator} when using online algorithms. Furthermore, we also specialize the results when the DCS is computed deterministically. 

\begin{restatable}{theorem}{est}\label{theo:dcs-variance}
Consider an online algorithm such that $M_T \ge 1$ holds almost surely. Then, for $t \in \{0, \dots, T-1 \}$, we define $f_t$ as:
\begin{align*}
f_t \coloneqq \gamma^{2t} \Var(R_t) + 2 \sum_{t'=t+1}^{T-1} {\gamma^{t+t'}} \Cov(R_t, R_{t'}).
\end{align*}  
Consider any random realization $\bm{M}$ of the online algorithm, and let $\tilde{J}_{\bm{M}}$ be an additional estimator that is computed by collecting a new dataset of size $\Lambda$ using $\bm{M}$. Then, it holds that:
\begin{align}\label{eq:dcs-variance}
    \textup{MSE} \left[ \hat{J}_{\bm{M}} \right] \le 2 \mathbb{E} \left[ \sum_{t=0}^{T-1} \frac{f_t}{N_t} \right] + 2 \mathbb{E}[(\hat{J}_{\bm{M}} - \tilde{J}_{\bm{M}} )^2].
\end{align}
Furthermore, in the case in which $\bm{M}$ is deterministic, i.e., $\bm{M} = \bm{m}$ for some fixed $\bm{m}$ such that $m_T \ge 1$ holds, we have that:
\begin{align}\label{eq:dcs-variance-det}
	\textup{MSE} \left[ \hat{J}_{\bm{m}} \right] = \Var(\hat{J}_{\bm{m}}) = \sum_{t=0}^{T-1} \frac{f_t}{n_t}.
\end{align}
\end{restatable}

Theorem \ref{theo:dcs-variance} provides an upper bound on the MSE of the estimator of Equation \eqref{eq:ttmc-estimator} when used with any arbitrary online algorithm that guarantees that $M_T \ge 1$ holds almost surely. Moreover, it also specializes the result when, instead of using an online algorithm, we use a deterministic DCS.

These results deserve some comments. For the clarity of the exposition, we start by analyzing the simpler case where the DCS is deterministic. First of all, we observe that, the MSE of the proposed estimator coincides with its variance, as for deterministic DCS the estimator is unbiased. Furthermore, Equation \eqref{eq:dcs-variance-det} expresses the variance of $\hat{J}_{\bm{m}}$ as the sum, over the different time steps $t$, of $\frac{1}{n_t}$ (i.e., the reciprocal of the number of samples collected under $\bm{m}$ at step $t$) multiplied by $f_t$ (i.e., the variance of the reward at step $t$ plus the covariances between $R_t$ and the rewards collected at future steps). Thus Equation \eqref{eq:dcs-variance-det} nicely captures the two examples that we provided at the beginning of this chapter.
\begin{example}[continues=exe:1]
When the reward is different from $0$ in the last interaction step only, the objective function reduces to $\frac{\gamma^{2(T-1)}}{n_{T-1}}\Var\left[ R_{T-1} \right]$, which is clearly minimized for the uniform strategy.
\end{example}
\begin{example}[continues=exe:2]
 Conversely, when the reward is different from $0$ in the first step only, we obtain $\frac{\Var\left[ R_{0} \right]}{n_0}$, meaning that the entire interaction budget should be dedicated to estimate $R_0$.
\end{example}
More generally, Equation \eqref{eq:dcs-variance-det} directly leads to a formulation of an optimal \emph{deterministic} DCS baseline for the problem. Specifically, given a budget $\Lambda$, we define the optimal deterministic DCS $\bm{n}^*$ as the solution of the following optimization problem:
\begin{equation}\label{sys:opt-prob}
\begin{aligned} 
\min_{\bm{n}} \quad & \sum_{t=0}^{T-1} \frac{1}{n_t} \left( \gamma^{2t} \Var(R_t) + 2 \sum_{t'=t+1}^{T-1} {\gamma^{t+t'}} \Cov(R_t, R_{t'}) \right)  \\
\textrm{s.t.} \quad &  \sum_{t=0}^{T-1} n_t \le \Lambda \\
  & n_t \ge n_{t+1}, \quad \forall t \in \{0, \dots, T-2\}  \\
  & n_t \in \mathbb{N}_+, \quad \forall t \in \{0, \dots, T-1\},
\end{aligned}
\end{equation}
where the constraints $n_t \ge n_{t+1}$ captures the sequential nature of the interaction with the environment. Before analyzing the more general case in which an online algorithm is used to output a DCS, we observe that \eqref{sys:opt-prob} highlights the pitfalls of pre-determined data collection strategies. Indeed, consider the two examples mentioned above. Before executing policy $\pi$, the agent cannot distinguish between the two different objective functions, i.e., $\frac{\gamma^{2(T-1)}}{n_{T-1}}\Var\left[ R_{T-1} \right]$ and $\frac{\Var\left[ R_{0} \right]}{n_0} $. As a consequence, any pre-determined schedule fails to adapt to the peculiar features of the domain at hand. From a general perspective, this is because the optimal strategy resulting from the optimization problem \eqref{sys:opt-prob} can be computed prior to the interaction with the environment only by an oracle that knows in advance the underlying reward process induced by the agent's policy $\pi$ in the MDP. Nevertheless, we also note that all the terms that appear in the objective function, i.e., $f_t$, can be estimated if some trajectories with the environment are available to the agent, meaning that algorithms that sequentially allocate the available budget $\Lambda$ (i.e., stochastic DCSs) might be able to successfully minimize Equation \eqref{eq:dcs-variance-det}.

We now focus on the case in which the resulting DCS is the output of an online algorithm. In this case, Theorem \ref{theo:dcs-variance} provides an upper bound on the MSE of $\hat{J}_{\bm{M}}$ which is composed of two terms. Given the interpretation that we provided for deterministic DCSs, the first term can be understood as the expected variance of an algorithm that constructs an estimator using a DCS $\bm{m}$ sampled a-priori (i.e., without actually interacting with the environment) from the distribution of the DCS returned by the online algorithm. The second term, instead, is given by $\mathbb{E}[(\hat{J}_{\bm{M}} - \tilde{J}_{\bm{M}})^2]$, where $\tilde{J}_{\bm{M}}$ is a "fictitious" estimator that is computed by collecting an additional dataset using the same DCS $\bm{M}$ which is given in output by the online algorithm. As the proof of Theorem \ref{theo:dcs-variance} reveals, this term directly arises from the fact that $\hat{J}_{\bm{M}}$ is a biased estimate of $J(\pi)$. 

Given these intuitions, in the rest of this chapter, we aim at developing an online algorithm that adapt to the problem-dependent features of the environment and the policy to evaluate. To this end, we propose to minimize the quantity $\mathbb{E} \left[ \sum_{t=0}^{T-1} \frac{f_t}{N_t}\right]$. Indeed, as we discussed above, this term captures several aspects of the problem at hand that cannot be tackled by using deterministic DCSs. Furthermore, given the results of Theorem \ref{theo:dcs-variance}, $\mathbb{E} \left[ \sum_{t=0}^{T-1} \frac{f_t}{N_t}\right]$ can be seen as a surrogate of the error of the final estimate, where the only component that is ignored is the bias of the resulting estimator.
}

\section{Robust and Iterative DCS Optimization}\label{sec:method}
%Given the findings presented in Section \ref{sec:opt-dcs}, we now proceed presenting our algorithmic solution. As we have already anticipated, selecting a DCS prior to the interaction with the environment does not allow the agent to adapt its trajectory schedule toward timesteps in which more data are required to minimize the estimation error. Indeed, as shown in Theorem \ref{theo:dcs-variance}, minimizing the estimator's variance requires exact knowledge on the problem being solved, which is not available to the agent, and the optimal solution (i.e., the minimum variance DCS) is clearly dependent on these unknown quantities. To solve this issue, in the rest of this section, we propose a novel approach, called \textbf{R}obust and \textbf{I}terative \textbf{D}ata collection strategy \textbf{O}ptimization (\rido), whose pseudocode can be found in Algorithm \ref{alg:rido}. The main intuition behind \rido lies in splitting the available budget $\Lambda$ into mini batches of interactions that will be allocated \emph{sequentially}. At each iteration, the agent determines the most convenient schedule of trajectories that optimizes a \emph{robust} upper bound of the objective function presented in \eqref{sys:opt-prob}, whose quality improves as the agent gathers more data. 

Given the findings of Section \ref{sec:opt-dcs}, we now present our algorithmic solution that aims at avoiding the highlighted pitfalls of pre-determined DCSs. Our approach is called \textbf{R}obust and \textbf{I}terative \textbf{D}ata collection strategy \textbf{O}ptimization (\rido), and its pseudocode is available in Algorithm \ref{alg:rido}.
%As one can expect from the last remarks of Section \ref{sec:opt-dcs}, 
The central intuition behind \rido lies in splitting the available budget $\Lambda$ into mini-batches of interactions that the agent will allocate \emph{sequentially}. At each iteration, the agent will compute the most convenient schedule of trajectories that optimizes an \emph{empirical and robust} version of the objective function presented in \eqref{sys:opt-prob}, whose quality improves as the agent gathers more data.
%\am{Con 'robust' vuoi dire qualcosa di diverso dal fatto che e' un upper bound in alta probabilita'? Valutiamo l'uso del termine 'robust'}

\begin{algorithm}[t]
\caption{Robust and Iterative DCS Optimization (\rido).} \label{alg:rido}
\begin{algorithmic}[1]
\small
\REQUIRE Interaction budget $\Lambda$, batch size $b$, robustness level $\beta$, policy $\pi$
\vspace{0.1cm}
\STATE{Collect $\mathcal{D}$ using policy $\pi$ and $\bm{\hat{N}}_0 = \left(\frac{b}{T}, \dots, \frac{b}{T} \right)$}\label{line:1}
\STATE{Set $K = \frac{\Lambda}{b}$ and initialize empirical estimates $\sqrt{\Empvar_1\left[R_t\right]}$ and $\Empcov_1\left[R_t, R_t'\right]$}\label{line:2}
\FOR{$i = 1, \dots, K-1$}
\STATE{Collect $\mathcal{D}_i$ using policy $\pi$ and $\bm{\hat{N}}_i$, where $\bm{\hat{N}}_i$ is computed solving problem \eqref{sys:emp-opt-prob}}\label{line:4}
\STATE{Update empirical estimates $\sqrt{\Empvar_i\left[R_t\right]}$ and $\Empcov_i\left[R_t, R_t'\right]$ using $\mathcal{D}_i$ and set $\mathcal{D} \leftarrow \mathcal{D} \cup \mathcal{D}_i$}\label{line:5}
\ENDFOR
\end{algorithmic}
\end{algorithm}

We now describe in-depth the behavior of the algorithm. For simplicity of exposition and analysis, we suppose that the size of the mini-batch $b$ is such that $b \mod T = 0$ and $b \ge 2T$. At the beginning (Lines~\ref{line:1}-\ref{line:2} in Algorithm \ref{alg:rido}), the agent spends the first mini-batch $\bm{\hat{N}_0}$ at collecting $\frac{b}{T}$ trajectories of length $T$ (i.e., the uniform approach). This preliminary collection phase is a starting round in which some initial experience is gathered to properly initialize estimates of relevant quantities used throughout the algorithm. More specifically, at each iteration $i$, the agent maintains empirical estimates of the unknown quantities that define the variance of the estimate, i.e., the standard deviation of the reward at step $t$, namely $\sqrt{\Empvar_i\left[R_t\right]}$, and the covariances between rewards at different steps, namely $\Empcov_i\left(R_t, R_{t'}\right)$. Then, at each round (Lines~\ref{line:4}-\ref{line:5} in Algorithm \ref{alg:rido}), the DCS of the current mini-batch $\bm{\hat{n}_i}$ is computed solving the optimization problem~\eqref{sys:emp-opt-prob} whose objective function is a robust estimate of the objective function of the original optimization problem \eqref{sys:opt-prob}. More specifically, at each round $i$, the agent aims at solving the following problem:
\begin{equation}\label{sys:emp-opt-prob}
\begin{aligned} 
\min_{\bm{N}} \quad & \sum_{t=0}^{T-1} \frac{1}{N_t} \left( \gamma^{2t} \left( \sqrt{\Empvar_i(R_t)} + \textrm{C}^\sigma_{i,t} \right)^2 + 2 \sum_{t'=t+1}^{T-1} {\gamma^{t+t'}} \left(\Empcov_i(R_t, R_{t'}) + \textrm{C}^{c}_{i,t,t'} \right) \right) \\
\textrm{s.t.} \quad &  \sum_{t=0}^{T-1} N_t \le b \\
  & N_t \ge N_{t+1}, \quad \forall t \in \{0, \dots, T-2\}  \\
  & N_t \in \mathbb{N}_+, \quad \forall t \in \{0, \dots, T-1\},
\end{aligned}
\end{equation}
where $\textrm{C}^\sigma_{i,t}$ and $\textrm{C}^{c}_{i,t,t'}$ are exploration bonuses for variances and covariances respectively, defined as:
\begin{equation}\label{eq:ci}
\textrm{C}^\sigma_{i,t} \coloneqq \sqrt{\frac{2 \log\left( \beta \right)}{\sum_{j=1}^{i-1} \hat{N}_{j,t}}}, \quad \quad \quad \quad \textrm{C}^{c}_{i,t,t'} \coloneqq 3 \sqrt{\frac{2 \log\left( \beta \right)}{\sum_{j=1}^{i-1} \hat{N}_{j,t'}}},
\end{equation}
where $\beta \ge 1$ is a hyper-parameter that specifies the amount of exploration used to solve the optimization problem, and $\hat{N}_{j,t}$ is the number of samples collected by \rido during phase $j$ at time step $t$.
%\am{Non si dice subito cosa sono $\hat{n}_{j,t}$} 
%\rp{Introdotto sia sopra durante la narrazione che esplicitamente dopo l'equazione.}
We now provide further explanations on the optimization problem \eqref{sys:emp-opt-prob} and Equation \eqref{eq:ci}. First of all, we notice how each term in the original objective function, namely $f_t$
%\amout{$f_t$ lo si vede qui la prima volta e non e' chiaro cosa rappresenta} \rpout{Introdotto in Theorem 1.}
, is replaced with its relative empirical estimation plus exploration bonuses, each of which is directly related to components within $f_t$, e.g., $\Var\left(R_t\right)$ is replaced with $\left( \sqrt{\Empvar_i(R_t)} + \textrm{C}^\sigma_{i,t} \right)^2$ and $\Cov\left(R_t, R_{t'}\right)$ is replaced with $\Empcov_i\left(R_t, R_{t'} \right) + \textrm{C}^{c}_{i,t,t'}$.
%\amout{Perche' l'esempio e' solo con la covarianza?} \rpout{Non mi piaceva molto esteticamente che con la varianze dovesse venire fuori una linea cosi' grossa.}. 
Intuitively, the purpose of the exploration bonus is to consider the uncertainty that arises from replacing exact quantities with their empirical estimation. This introduces in \rido a source of robustness w.r.t. the noise that is intrinsically present in the underlying estimation process. At this point, concerning the shape of Equation \eqref{eq:ci}, focus for the sake of exposition on $\textrm{C}^\sigma_{i,t}$. First of all, we notice that the hyper-parameter $\beta$ governs the robustness which is taken into account while replacing $\Var\left( R_t \right)$ with its empirical estimate. Larger values of $\beta$, correspond, indeed, to larger $\textrm{C}^\sigma_{i,t}$, and, consequently, a higher level of robustness w.r.t. the uncertainty. Furthermore, as we can notice, $\textrm{C}^\sigma_{i,t}$ decreases with the number of samples collected in the previous iterations at step $t$, i.e., $\sum_{j=1}^{i-1} \hat{N}_{j,t}$. This quantity coincides with the number of samples that are used to estimate $\sqrt{\Empvar_{i} \left[R_t \right]}$.\footnote{Similar comments apply to $\textrm{C}^{c}_{i,t,t'}$ as well. The only difference stands in the fact that to estimate the empirical covariance between two subsequence steps $t$ and $t'$, samples up to time $t'$ are required. For this reason, the denominator implies the summation of the number of samples gathered at $t'$ over the previous iterations.} This formulation captures the following aspect: more data is available to the agent to estimate $\Var\left( R_t \right)$, the more accurate its estimate will be, and, consequently, its exploration bonus will shrink to $0$.  As one can expect, with this approach, the quality of the objective function used in \rido increases with the number of iterations. Consequently, the agent will progressively adapt the mini-batch DCS toward time steps where more data is required to minimize the surrogate of the policy return error. We conclude with two remarks. First, we notice that \rido can be applied with $\gamma=1$, as it does not deeply rely on the property of discounted sums.  
Secondly, the optimization problem \eqref{sys:emp-opt-prob} is a complex integer and non-linear optimization problem. Before diving into the statistical analysis of \rido, we discuss how to solve \eqref{sys:emp-opt-prob} in the next section.

\subsection{Solving the Empirical Optimization Problem}\label{sec:emp-problem}

As noticed above, directly solving problem \eqref{sys:emp-opt-prob} requires significant effort since it is an integer, non-linear optimization problem. In this section, we discuss how to overcome these challenges.

We first perform a \emph{continuous relaxation}, replacing the integer constraint $N_t \in \mathbb{N}_+$ with  $N_t \ge 1$.  Once a solution $\bar{\bm{N}}^*$ to the relaxed optimization problem is found, it is possible to obtain a proper (i.e., integer) DCS by flooring each $\bar{n}_t^*$ and allocating the remaining budget uniformly. As we shall see, this approximation introduces constant terms in the theoretical guarantees of \rido only.
At this point, the resulting optimization problem is a non-linear problem that, unfortunately, is generally non-convex. This issue occurs when the following condition is verified for some time step $t$:
\begin{align}\label{eq:ill-condition-example}
f_t = \gamma^{2t} \left( \sqrt{\Empvar_i(R_t)} + \textrm{C}^\sigma_{i,t} \right)^2 + 2 \sum_{t'=t+1}^{T-1} {\gamma^{t+t'}} \left(\Empcov_i(R_t, R_{t'}) + \textrm{C}^{c}_{i,t,t'} \right) < 0.
\end{align}
To solve this challenge and make \rido computationally efficient, we develop an approach based on a hidden property of the original optimization problem \eqref{sys:opt-prob}. More specifically, we start by noticing that even the continuous relaxation of \eqref{sys:opt-prob} is non-convex since $f_{\bar{t}} < 0$ 
might occur, for some $\bar{t} \in \left\{0, \dots, T-2 \right\}$, in the presence of negative covariances with future steps. In this case, however, since $\sum_{t=\bar{t}}^{T-1} f_t$ represents a proper variance, which is always non-negative, there always exists $t' > \bar{t}$ such that $\sum_{t={\bar{t}}}^{t'} f_t \ge 0$. Furthermore, it is possible to show that the optimal solution of the relaxed optimization problem is uniform in the interval $\left\{\bar{t}, \dots, t' \right\}$, namely $n_{\bar{t}}^* = n_{\bar{t}+1}^* = \dots = n^*_{t'}$ (proof in Appendix \ref{app:proofs}).
For this reason, it is possible to define a \emph{transformation} of the optimization problem that preserves the optimal solution, in which the variables $n_{\bar{t}}, \dots, n_{t'}$ are replaced with a single variable $y$. The objective function is modified accordingly, namely $\frac{f_{\bar{t}}}{n_{\bar{t}}} + \dots + \frac{f_{t'}}{n_{t'}}$ is replaced with $\frac{f_{\bar{t}} + \dots + f_{t'}}{y}$ in the objective function. A visualization of the transformation is proposed in Figure~\ref{fig:vis}.
By repeating the procedure for all the negative $f_t$, we obtain a transformation of the original problem which is now convex. Once the solution to this convex transformed optimization problem is found, one can quickly recover the relaxed DCS in its $T$-dimensional form. 

\begin{figure}
  \centering
  \includegraphics[width=10cm]{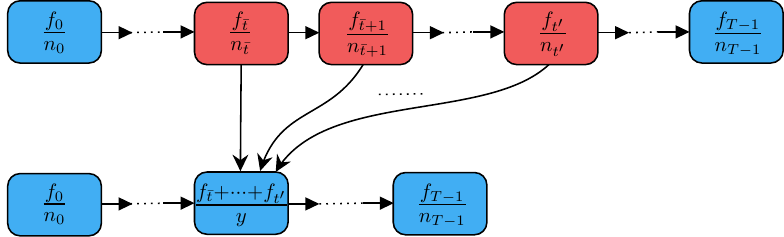} 
  \caption{Visualization of the transformation between the optimization problems. The first row shows the objective function of the original optimization problem, while the second one its transformation.}
  \label{fig:vis}
\end{figure}

Building on these results, we apply in \rido a similar procedure that transforms the relaxed version of \eqref{sys:emp-opt-prob} into a new problem where the negative time steps (i.e., steps in which Equation \eqref{eq:ill-condition-example} holds) are ``grouped'' with future time steps as long as the total summation is positive. In this way, (i) the resulting optimization problem is convex and (ii) as our analysis will reveal, this procedure has no impact on the theoretical properties of \rido (i.e., the result is the same as assuming access to an oracle that can solve non-linear and non-convex problems). As a concluding remark, we refer the reader to Appendix \ref{app:proofs} for a formal description of the above-mentioned procedure.

%\amout{Non e' chiaro perche' la (ii). Nel senso che se la soluzione del problema trasformato e' la stessa,  %\amout{Proprio 'the same'?} \rpout{Eh, in teoria si. In alta probabilita' vale l'equivalenza di cui sopra che ci consente di trovare l'ottimo.}
%perchè potrebbe succedere qualcosa nella teoria...} \rpout{Eh, il motivo vero e' che questo vale solo in alta probabilita' quando si usano delle stime. Formalmente quello che usiamo noi e' un'euristica ricavata da questo. Questo e' collegato a quello che c'e' scritto sotto.} 

%\rp{Da notare che in quest'ultima parte c'e un po' di imprecisione, in quanto il termine successivo positivo potrebbe non esistere. O, ancora peggio, potrebbe esistere al primo $t$ per cui (7) vale, ma non per l'ultimo. A me sembrava un compromesso sensato pero' fra il descrivere la procedura in modo molto preciso (ma probabilmente poco chiaro) e la verita'. Che dite?}
%\am{Va bene per me}

\subsection{Theoretical Analysis}\label{sec:theo-analysis}

{\color{black}

In this section, we discuss the theoretical properties of \rido. First, we observe that using, \rido, Equation \eqref{eq:ttmc-estimator} is a consistent estimator. Indeed, a sufficient condition for the consistency is that $\lim_{\Lambda \rightarrow + \infty} {\Lambda}^T \exp(-C N_{T}) \rightarrow 0$ holds almost surely for any constant $C > 0$. Nevertheless, in each phase $i$, \rido allocates, by definition, at least $1$ sample to $N_{T,i}$. Consequently, for any constant mini-batch size $b$, it holds that:
\begin{align*}
\lim_{\Lambda \rightarrow + \infty} {\Lambda}^T \exp(-C N_{T}) & \le \lim_{\Lambda \rightarrow + \infty} {\Lambda}^T \exp(-C K) \\ & = \lim_{\Lambda \rightarrow + \infty} {\Lambda}^T \exp\left(- \frac{C \Lambda}{b}\right) \\ & = 0.
\end{align*}
That being said, the following result summarizes our main finding.

\begin{restatable}{theorem}{iterativeregret}\label{theo:iterative-regret}
Let $\bm{n}^*$ be an optimal solution of \eqref{sys:opt-prob}, $b \ge 2T$ and $\beta = \frac{6(T+T^2)\Lambda K}{\delta}$. Suppose that $\sum_{t=0}^{T-1} f_t > 0$ and consider the DCS $\bm{N}$ computed by \rido. Then, with probability at least $1-\delta$ it holds that:
\begin{align}\label{eq:regret-bound}
    \sum_{t=0}^{T-1} \frac{f_t}{N_t} \le \mathcal{O}\left( \Var \left[ \hat{J}_{\bm{{n}^*}} \right] + \frac{2b}{\Lambda} \sum_{t: f_t < 0} |f_t| \right).
\end{align}
\end{restatable}
Theorem \ref{theo:iterative-regret} presents a high-probability upper bound on the surrogate objective function that presented in Section \ref{sec:opt-dcs}. Plugging this result into Theorem \ref{theo:dcs-variance}, we obtain a high-probability upper-bound on the  MSE of estimator of Equation \eqref{eq:ttmc-estimator} when the considered DCS is computed using \rido. We observe that Equation \eqref{eq:regret-bound} is composed of two terms. The first one is the variance of the optimal deterministic DCS computed as in \eqref{sys:opt-prob}, while the second one, instead, is related to the negative terms that are possibly present in the objective function. Among the two terms, this one has the worst dependence on $\Lambda$. We are unsure whether this term is an artifact of our analysis, a sign of the possible sub-optimality of the algorithm, or a fundamental challenge of the setting.\footnote{We highlight that $f_t < 0$ makes the objective function non-convex.} We leave closing this gap to future work. Nevertheless, it is worth mentioning that, whenever $f_t \ge 0$ holds for all time steps $t$, for sufficiently large budget of $\Lambda$, we have that:
\begin{align*}
\sum_{t=0}^{T-1} \frac{f_t}{N_t} \le \mathcal{O} \left( \Var\left[ \hat{J}_{\bm{n^*}}(\pi) \right] \right) = \mathcal{O} \left( \sum_{t=0}^{T-1} \frac{f_t}{n_t^*}\right).
\end{align*}
Consequently, in these cases, the surrogate of the objective function of the returned DCS computed by \rido is proportional, up to constant multiplicative factors, to the one of the variance of the optimal deterministic DCS. In this sense, \rido is adapting to the problem-dependent feature of the underlying task.\footnote{In Appendix~\ref{app-subsec:tech-details}, we also prove that this result is, in general, not possible to achieve for pre-determined schedules, i.e., for existing pre-determined algorithms $\sum_{t=0}^{T-1} \frac{f_t}{n_t}$ is not proportional up to constant multiplicative factor to the variance of the optimal deterministic baseline.} 
}

\section{Numerical Validation}\label{sec:exp}
In this section, we propose numerical validations that aim at assessing the empirical performance of \rido. More specifically, we focus on the comparison between our approach, the classical uniform-in-the-horizon strategy, and the robust DCS by \citet{poiani2022truncating}. We report the results across multiple domains, values of budget $\Lambda$, and discount factor $\gamma$. {\color{black} To measure the performance of all algorithms, we first estimated the return of the target policy to be evaluated using a large number of uniform trajectories (i.e., $1000$). Then, we measured the empirical MSE of each baseline at the end of data collection process over 1000 independent runs. Before discussing our results in detail, we describe our experimental settings in depth. We notice that for each considered value of $\Lambda$, the experiment is repeated (i.e., we do not use data collected with smaller $\Lambda$'s). To conclude, we refer the reader to Appendix \ref{app:experiments} for further details on the experiments (e.g., ablations, additional results, hyper-parameters, visualizations of the resulting DCSs).}

\textbf{Experimental Setting}~~In our experiments, we consider the following four domains. We start with the Inverted Pendulum \citep{gym}, a classic continuous control benchmark, where the agents' goal is to swing up a suspended body and keep it in the vertical direction. We, then, continue with the Linear Quadratic Gaussian Regulator \citep[LQG,][]{curtain1997linear}, where the agent controls a linear dynamical system with the objective of reducing a total cost that is expressed as a quadratic function. Then, we consider a $2$D continuous navigation problem, where an agent starts at the bottom left corner of a room and needs to reach a goal region in the upper right corner. The agent receives reward $0$ everywhere except inside the goal area, where the reward is positive and sampled from a Gaussian distribution. Finally, we consider the Ant environment from the MuJoCo \citep{todorov2012mujoco} suite, where the agent controls a four-legged $3$D robot with the goal of moving it forward. Further domain details are provided in Appendix \ref{app:experiments}. Concerning the policy that we evaluate for the Inverted Pendulum and the Ant, we rely on pre-trained deep RL agents made publicly available by \citet{rlzoo3}. For the LQG, instead, we evaluate the optimal policy that is available in closed form by solving the Riccati equations, and, finally, for the $2$D navigation task, we roll out a hand-designed policy that minimizes the distance of the agent's position w.r.t. to the center of the goal region. 

\textbf{Results}~~Figure~\ref{fig:mainresults} reports the results varying the discount factor and the available budget. The second row is obtained under the same experimental setting as the first one, but with lower values of $\gamma$. Let us first focus on the sub-optimality of the non-adaptive DCSs (i.e., the uniform strategy and the robust one of \citet{poiani2022truncating}). Indeed, as suggested by Theorem \ref{theo:dcs-variance}, being computed prior to the interaction with the environment, these algorithms cannot adapt the collection of samples to minimize the error of the return estimator. This is clear by looking, for instance, at the results of Continuous Navigation and the LQG. Indeed, in the Continuous Navigation domain, the reward is sparse and received close to the end of the estimation horizon $T$. In this scenario, the robust DCS blindly truncates trajectories, thus, avoiding the collection of experience in the most relevant timesteps. Conversely, in the LQG experiments, the optimal policy that arises from the Riccati equation pays a stochastic control cost\footnote{The uncertainty, in this case, arises both from the noise of the system together with the stochasticity of the initial state distribution.} at the beginning of the estimation horizon to bring the state of the system close to stability, after which the reward will remain almost constant. In this case, the uniform DCS results in a highly sub-optimal behavior as most of the estimation uncertainty is related to the initial interaction steps. \rido, on the other hand, thanks to its adaptivity, is able to obtain the best results in both domains. Indeed, in the Continuous Navigation problem, it achieves the same performance level as the uniform strategy, while in the LQG it even outperforms the robust DCS of \citet{poiani2022truncating}. The reason is that \citet{poiani2022truncating} truncates trajectories solely depending on the value of $\gamma$, and, therefore, it might waste a portion of its budget in trajectories of sub-optimal length, while \rido, since it aims at minimizing the error of the final estimation, is able to focus the collection of data in the most convenient way. Similar comments to those made for the LQG hold for the Pendulum domain as well. Concerning the Ant environment, instead, we notice that for $\gamma=0.999$ there is no significant difference between any of the presented schedules. Interestingly, however, as soon as we decrease $\gamma$ to $0.99$, we can appreciate the sub-optimality of the uniform strategy, which wastes a portion of its budget in gathering samples that are significantly discounted, and, therefore, their weight in the estimator's error shrinks to $0$. On the other hand, the robust strategy and \rido avoid this pitfall thanks to the exploitation of the discount factor, thus obtaining reduced error estimates. Finally, we remark that \rido has achieved the most competitive performance across various domains, values of the discount factor, and budget, thus clearly highlighting the benefits of adaptive strategies w.r.t. pre-determined ones.   

\begin{figure}
  \centering
  \includegraphics[width=13.5cm]{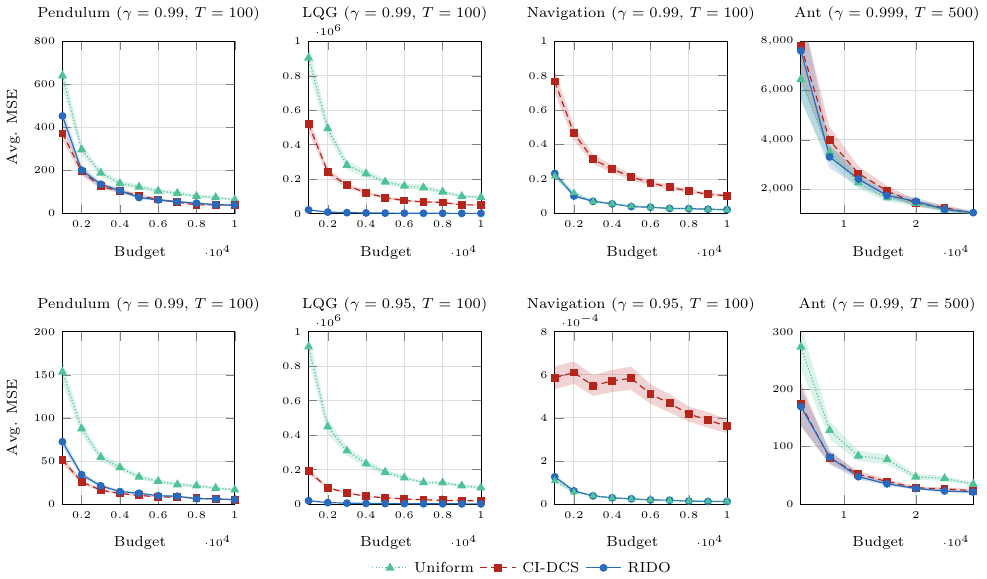} 
  \caption{Empirical MSE (mean and $95$\% confidence intervals over $100$ runs) on the considered domains and baselines. The first row considers higher values of $\gamma$ w.r.t. the second one.}
  \label{fig:mainresults}
\end{figure}

\section{Related Works}\label{app:rel-works}

We now present a comprehensive analysis and discussion of previous works that are closely connected to our own research. First of all, our work focuses on estimating a policy's performance in a given MDP \citep{sutton2018reinforcement}. Considering the significance of this task, reducing the the error of the return estimator, is a problem that has received significant attention in the literature. A vast family of approaches that can be used to solve this problem deeply exploits the Markovian properties of the environment by relying on Temporal Difference \citep[TD, see, e.g.,][]{singh1996reinforcement,sutton1988learning,lee2019target,riquelme2019adaptive,qu2019nonlinear} learning. On the other hand, our work focuses purely on Monte Carlo simulation, which can be transparently applied to non-Markovian environments. Another relevant line of work deals with optimizing the agent's policy to collect data within an environment (i.e., \emph{behavioral} policy) to reduce the variance of an unbiased estimator for the return of a different \emph{target} policy \citep{hanna2017data,zhong2022robust,mukherjee2022revar}. These techniques are referred to as off-policy evaluation methods and usually rely on Importance Sampling \citep[e.g,][]{hesterberg1988advances,owen2013monte} techniques to guarantee the unbiasedness of the resulting estimate. However, these studies significantly differ from ours in that, instead of aiming for a behavior policy that reduces the estimator variance, our goal is to directly exploit the properties of Monte Carlo data collection to reduce the on-policy estimator error. 

% Exploration
In the context of RL, exploration bonuses are widely adopted in control (where the goal is learning an optimal policy) to tackle the exploration-exploitation dilemma \citep[e.g.,][]{brafman2002r,auer2008near,tang2017exploration,jin2018q,o2018uncertainty,zanette2019tighter}. Initially, when the agent has limited knowledge about the environment, the exploration bonuses drive it to explore widely. As the agent's knowledge improves, the exploration bonuses decrease, and the agent can shift towards exploiting its learned policy more. In our work, instead, we use exploration bonuses to introduce a source of robustness w.r.t. the objective function that we are interested in.

% Truncating Trajectories 
Finally, the work that is most related to ours is \citet{poiani2022truncating}, where, the concept of truncating trajectories has been analyzed in the context of Monte Carlo RL. More specifically, the authors derived a \emph{non-adaptive} schedule of trajectories that provably minimizes confidence intervals around the return estimator. In this work, on the other hand, we have shown the sub-optimality of pre-determined schedules, and we designed an \emph{adaptive} algorithm that aims at minimizing the error of the final estimate. The concept of truncating trajectories has also received some attention in other fields of research such as model-based policy optimization \citep{nguyen2018improving,janner2019trust,bhatia2022adaptive,zhang2023uncertainty}, multi-task RL \citep{farahmand2016truncated} and imitation learning \citep{sun2018truncated}. However, in all these works, the motivation, the method, and the analysis completely differ w.r.t. what has been considered here. Finally, the concept of truncating trajectories in Monte Carlo RL drew inspiration from a recent work in the field of multi-fidelity bandit \citep{poiani2022multi}, where the authors considered the idea of cutting trajectories while interacting with the environment to obtain a biased estimate of the return of a policy in planning algorithms such as depth-first search.

\section{Conclusions and Future Works}\label{sec:conclusions}
In this work, we studied the problem of allocating a budget $\Lambda$ of transitions in the context of Monte Carlo policy evaluation to reduce the error of the policy expected return estimate. Leveraging the formalism of Data Collection Strategy (DCS) to model how an agent spends its interaction budget, we started by analyzing the error of a return estimator for any possible DCS. Our result reveals that DCSs determined prior to the interaction with the environment (e.g., the usual uniform-in-the-horizon one and the robust one of \citet{poiani2022truncating}) fail to satisfy the ultimate goal of policy evaluation, i.e., produce a low error estimate. Furthermore, it also suggests that algorithms that spend the available budget $\Lambda$ iteratively might successfully adapt their strategy to minimize the error of the return estimator. Inspired by these findings, we propose an \emph{adaptive} method, \rido, that, by exploiting information that has already been collected, can dynamically adapt its DCS to allocate a larger portion of transitions in time steps in which more accurate sampling is required to reduce the error of the final estimate. After conducting a theoretical analysis on the properties of the proposed method, we present empirical studies that confirm its adaptivity across a different number of domains, values of budget $\Lambda$, and discount factors $\gamma$.

Our study offers exciting possibilities for future research. For example, it would be interesting to extend our ideas to policy search algorithms (e.g., \citet{williams1992simple}), with the goal of finding DCSs that minimize the error of the empirical gradient that is adopted in the update rule. Furthermore, we notice that, since our approach is purely based on MC simulation, it does not fully leverage the Markovian properties of the underlying MDP. Combining TD techniques \citep{sutton2018reinforcement} with mechanisms that truncate trajectories is a challenging and open research question that could lead to further improvements in the efficiency of RL algorithms.

\section*{Acknowledgements}
This paper is supported by PNRR-PE-AI FAIR project funded by the NextGeneration EU program.

\bibliographystyle{plainnat}
\bibliography{biblio}

%%%%%%%%%%%%%%%%%%%%%%%%%%%%%%%%%%%%%%%%%%%%%%%%%%%%%%%%%%%%
\newpage
\appendix

\section{Proofs and Derivations}\label{app:proofs}

In this section, we provide complete proofs of our theoretical results. Specifically, Section \ref{app-subsec:consistency} contains the proof of Theorem \ref{theo:consist}, Section \ref{app-subsec:dcs-variance} the proof of Theorem \ref{theo:dcs-variance}; Section \ref{app-subsec:iter-regret} the proof of Theorem \ref{theo:iterative-regret}, and Section \ref{app-subsec:tech-details} proofs and details of additional statements that have been made in the main text (i.e., formal description of the transformation between optimization problems and how we applied this technique in \rido, difficulties in deriving closed-form solutions for the optimization problems of interest, sub-optimality examples of non-adaptive methods whose variance cannot scale with the variance of the optimal DCS.

\subsection{Proof of Theorem \ref{theo:consist}}\label{app-subsec:consistency}

\consistency*
{\color{black}
\begin{proof}
Fix any $\Lambda \in \mathbb{N}$ and $\epsilon > 0$. Denote by $\mathcal{V}$ the space of all DCSs vectors with total budget $\Lambda$ such that, for all $\bm{n} \in \mathcal{V}$, it holds that $\mathbb{P}(\{ \bm{N} = \bm{n} \}) > 0$. Then, we have that:
\begin{align*}
\mathbb{P}(|\hat{J}_{\bm{N}}(\pi) - J(\pi)| > \epsilon) & =  \sum_{\bm{n} \in \mathcal{V}} \mathbb{P}\left( \{ | \hat{J}_{\bm{n}}(\pi) - J(\pi) > \epsilon \} \textup{ and } \{\bm{N} = \bm{n} \} \right) \\ & \le \sum_{\bm{n} \in \mathcal{V}} 2\exp \left( \frac{-2 \epsilon^2}{\sum_{h=0}^{T} m_h \left( \sum_{t=0}^{T-1} \frac{\gamma^t}{n_t} \right)^2} \right) \\ & \le \sum_{\bm{n} \in \mathcal{V}} 2\exp \left( \frac{-2 \epsilon^2 n_T}{\gamma^T} \right) \\ & \coloneqq 2 \sum_{\bm{n} \in \mathcal{V}} \exp\left( -\tilde{C} n_T\right) \\ & \le 2 \max_{\bm{n} \in \mathcal{V}} \Lambda^T  \exp\left( -\tilde{C} n_T\right),
\end{align*}
where (i) the first inequality follows from the Hoeffding inequality (indeed, $\hat{J}$ for a fixed $\bm{n}$ is a sum of sub-gaussian random variables), and (ii) the third inequality follows by noticing that $|\mathcal{V}| \le \Lambda^T$.
Finally, the proof follows by taking the limit for $\Lambda \rightarrow +\infty$ and using Equation \eqref{eq:consistency} together with the fact that $\tilde{C} > 0$.
\begin{align*}
\end{align*}
\end{proof}
}

\subsection{Proof of Theorem \ref{theo:dcs-variance}}\label{app-subsec:dcs-variance}

\est*
{\color{black}
\begin{proof}
With simple algebraic manipulations, we have that:
\begin{align}
\textup{MSE} \left[ \hat{J}_{\bm{M}} \right] & = \E[ (\hat{J}_{\bm{M}} - J(\pi))^2 ] \\ & \le 2 \E[ (\tilde{J}_{\bm{M}} - J(\pi))^2 ] + 2 \E[ (\hat{J}_{\bm{M}} - \tilde{J}_{\bm{M}})^2 ]\label{eq:temp}.
\end{align}
At this point, focus on $\E[ (\tilde{J}_{\bm{M}} - J(\pi))^2 ]$. We begin by noticing that $\tilde{J}_{\bm{M}}$ is an unbiased estimate of $J(\pi)$. Indeed, we have that:
\begin{align*}
\mathbb{E}[\tilde{J}_{\bm{M}}] & = \mathbb{E} \left[ \sum_{h=1}^T \sum_{i=1}^{M_h} \sum_{t=0}^{h-1} \gamma^t \frac{R_t}{N_t} \right] \\ & = \mathbb{E}_{\bm{M}} \left[ \mathbb{E}\left[ \sum_{h=1}^T \sum_{i=1}^{M_h} \sum_{t=0}^{h-1} \gamma^t \frac{R_t}{N_t} | \bm{M} \right] \right] \\ & = \mathbb{E}_{\bm{M}} \left[ \sum_{h=1}^T M_h \sum_{t=0}^{t-1} \mathbb{E}\left[ \gamma^t \frac{R_t}{N_t} | \bm{M} \right] \right] \\ & = \mathbb{E}_{\bm{M}} \left[ \sum_{t=0}^{T-1} N_t \mathbb{E}\left[ \gamma^t \frac{R_t}{N_t} | \bm{M} \right] \right] \\ & = J(\pi),
\end{align*}
where (i) the second equality follows from the law of total expectation, (ii) the second one follows from the fact that the different trajectories are independent, and (iii) the third one from Theorem B.1 in \cite{poiani2022truncating}.\footnote{We note that having collected a \emph{new} dataset using a DCS $\bm{M}$ that is the output of an online algorithm, breaks the correlation in Estimator \eqref{eq:ttmc-estimator}. As a result, $\tilde{J}_{\bm{M}}$ in an unbiased estimator. On the other hand, $\hat{J}_{\bm{M}}$ is not unbiased since the length of the trajectory depends on the previous rewards, i.e., there is a correlation between the observed rewards and the variables $N_T$'s and $M_H$'s.}
Therefore, we have that $\tilde{J}_{\bm{M}}$, $\E[ (\tilde{J}_{\bm{M}} - J(\pi))^2 ] = \Var[\tilde{J}_{\bm{M}}]$, and, moreover, due to the law of total variance, we have that:
\begin{align*}
\Var[\tilde{J}_{\bm{M}}] & = \mathbb{E}_{\bm{M}}\left[ \Var(\tilde{J}_{\bm{M}} | \bm{M}) \right] + \Var_{\bm{M}}\left( \mathbb{E}\left[ \tilde{J}_{\bm{M}} | \bm{M} \right] \right) \\ & = \mathbb{E}_{\bm{M}}\left[ \Var(\tilde{J}_{\bm{M}} | \bm{M}) \right],
\end{align*}
where the second step follows from the fact that $\mathbb{E}[\hat{J}_{\bm{M}}] = J$. Finally, it remans to analyze $\mathbb{E}_{\bm{M}}\left[ \Var(\tilde{J}_{\bm{M}} | \bm{M}) \right]$. Specifically, we have that: 
\begin{align}\label{eq:var-proof-eq1}
\mathbb{E}_{\bm{M}}\left[ \Var(\tilde{J}_{\bm{M}} | \bm{M}) \right] & = \mathbb{E}_{\bm{M}} \left[ \sum_{h=1}^{T} M_h \Var\left( \sum_{t=0}^{h-1} \frac{\gamma^t R_t}{N_t} | \bm{M} \right)  \right],
\end{align}
where the equality follows from the fact that the, given $\bm{M}$, the different trajectories are independent. Now, from the variance of the sum of dependent random variables, we can further rewrite the right-hand side of Equation \eqref{eq:var-proof-eq1} as follows:
{\fontsize{10}{10.5}
\begin{align*}
\mathbb{E}_{\bm{M}} \left[ \sum_{h=1}^{T} M_h  \sum_{t=0}^{h-1} \Var\left(\frac{\gamma^t R_t}{N_t} | \bm{M} \right) \right] + \mathbb{E}_{\bm{M}} \left[ \sum_{h=1}^{T} M_h \sum_{t=0}^{h-2} \sum_{t'=t+1}^{h-1} 2 \Cov \left(\frac{\gamma^t R_t}{N_t}, \frac{\gamma^{t'} R_{t'}}{N_{t'}} | \bm{M} \right) \right].
\end{align*}
}
This, in turn, can be further rewritten as:
{\fontsize{10}{10.5}
\begin{align*}
\mathbb{E}_{\bm{M}} \left[ \sum_{h=1}^{T} {M_H}  \sum_{t=0}^{h-1} \frac{\gamma^{2t}}{N_T^2} \Var\left({R_t} | \bm{M} \right) \right] + \mathbb{E}_{\bm{M}} \left[ \sum_{h=1}^{T} M_h \sum_{t=0}^{h-2} \sum_{t'=t+1}^{h-1} 2 \frac{\gamma^{t+t'}}{N_t N_{t'}} \Cov \left({R_t}, R_{t'} | \bm{M} \right) \right].
\end{align*}
}
Indeed, for three random variables $X,Y,Z$, and for two scalars $a,b$, it holds that $\Var(aXY|Y) = a^2Y^2 \Var(X)$ and $\Cov(aXY, bYZ | Y) = abY \Cov(X,Z)$. 
Furthermore, we note that $\Var(R_t | \bm{M}) = \Var(R_t)$, and $\Cov(R_t, R_{t'} | \bm{M}) = \Cov(R_t, R_{t'})$, thus leading to:
{\fontsize{10.5}{11}
\begin{align*}
\mathbb{E}_{\bm{M}} \left[ \sum_{h=1}^{T} {M_H}  \sum_{t=0}^{h-1} \frac{\gamma^{2t}}{N_T^2} \Var\left({R_t} \right) \right] + \mathbb{E}_{\bm{M}} \left[ \sum_{h=1}^{T} M_h \sum_{t=0}^{h-2} \sum_{t'=t+1}^{h-1} 2 \frac{\gamma^{t+t'}}{N_t N_{t'}} \Cov \left({R_t}, {R_{t'}} \right) \right].
\end{align*}
}
 
At this point, focus on $\mathbb{E}_{\bm{M}} \left[ \sum_{h=1}^{T} {M_H}  \sum_{t=0}^{h-1} \frac{\gamma^{2t}}{N_T^2} \Var\left({R_t} \right) \right]$,
and fix $\bar{t} \in \{0, \dots, T-1 \}$. Unrolling the summation, we observe that its contribution appears in all $h$ such that $h > \bar{t}$. Thus, we obtain:
\begin{align*}
\mathbb{E}_{\bm{M}} \left[ \sum_{h=1}^{T} M_h \sum_{t=0}^{h-1} \frac{\gamma^{2t}}{N_t^2} \Var(R_t) \right].
 = \mathbb{E}_{\bm{M}} \left[ \sum_{t=0}^{T-1} \frac{\gamma^{2t}}{N_t^2} \Var(R_t) \sum_{h=t+1}^{T}M_h \right]
\end{align*}
Furthemore, since $\sum_{h=t+1}^{T} M_h = N_t - N_{t+1} + N_{t+1} - N_{t-2} + \dots + N_{T-2} - N_{T-1} + N_{T-1} = N_t$, we have that:
\begin{align}\label{eq:var-theo-eq1}
\mathbb{E}_{\bm{M}} \left[ \sum_{h=1}^{T} {M_H}  \sum_{t=0}^{h-1} \frac{\gamma^{2t}}{N_T^2} \Var\left({R_t} \right) \right] = \mathbb{E}_{\bm{M}}\left[\sum_{t=0}^{T-1} \frac{\gamma^{2t}}{N_t} \Var(R_t) \right].
\end{align}

Now, let us focus on $\mathbb{E}_{\bm{M}} \left[ \sum_{h=1}^{T} M_h \sum_{t=0}^{h-2} \sum_{t'=t+1}^{h-1} 2 \frac{\gamma^{t+t'}}{N_t N_{t'}} \Cov \left({R_t}, {R_{t'}} \right) \right]$. Fix an index $\bar{t} \in \left\{0, \dots, T-2 \right\}$ for the outer summation over time. Unrolling the summation, we observe that its contribution appears only in all $h$ such that $h > \bar{t} + 1$, therefore, we can rewrite this expectation as:
\begin{align*}
	\mathbb{E}_{\bm{M}} \left[ \sum_{t=0}^{T-2} \sum_{h=t+2}^T M_h \sum_{t'=t+1}^{h-1} \frac{2\gamma^{t+t'}}{N_t N_{t'}} \Cov \left(R_t, R_{t'} \right) \right].
\end{align*}
Now, fix $\bar{t} \in \{0, \dots, T-2\}$ as the index of the outer summation, and consider $t' \ge \bar{t} + 1$. Here, we note that $t'$ only appears for $h > t'$. Thus, by rearrenging the terms and noticing that $\sum_{h=t'+1}^T M_h = N_{t'}$, we have that:
{\fontsize{9}{9.5}
\begin{align}\label{eq:var-theo-eq6}
    \mathbb{E}_{\bm{M}} \left[ \sum_{t=0}^{T-2}  \sum_{t'=t+1}^{T-1} \frac{2\gamma^{t+t'}}{N_t N_{t'}} \Cov \left(R_t, R_{t'} \right) \sum_{h=t'+1}^{T} M_h \right] = \mathbb{E}_{\bm{M}} \left[ \sum_{t=0}^{T-2}  \sum_{t'=t+1}^{T-1} \frac{2\gamma^{t+t'}}{N_t} \Cov \left(R_t, R_{t'} \right) \right].  
\end{align}
}
Combining Equation \eqref{eq:var-theo-eq1} and \eqref{eq:var-theo-eq6}, we obtained:
\begin{align}\label{eq:var-theo-eq7}
\Var[\tilde{J}_{\bm{M}}] = \mathbb{E}\left[ \sum_{t=0}^{T-1} \frac{f_t}{N_t} \right].
\end{align}
Plugging this result within Equation \eqref{eq:temp} concludes the proof of Equation \eqref{eq:dcs-variance}. The proof of Equation \eqref{eq:dcs-variance-det}, instead, is a direct consequence of Equation \eqref{eq:var-theo-eq7}. Indeed, in this case, $\bm{M} = \bm{m}$, and $\textup{MSE}[\tilde{J}_{\bm{m}}] = \Var(\tilde{J}_{\bm{m}})$. 
\end{proof}
}

\subsection{Proof of Theorem \ref{theo:iterative-regret}}\label{app-subsec:iter-regret}

To prove Theorem \ref{theo:iterative-regret}, we first provide some preliminaries lemmas on the properties of the optimization problems that we are considering, toghether with some technical results that will be used in our proofs. Then, we will move towards the analysis of \rido.

\subsubsection{Preliminaries for the proof of Theorem \ref{theo:iterative-regret}}
We begin by proving the fact that for any timestep $t$ in which $f_t$ is negative, that there exists some future timestep $t'$ such that $\sum_{i=t}^{t'} f_i \ge 0$.

\begin{lemma}[Variance function property]\label{lemma:var-funct-property}
Consider $f_t = \gamma^{2t} \Var{R}_t + 2 \sum_{t'=t+1}^{T-1} \gamma^{t+t'} \Cov(R_t, R_{t'})$. For any $t \in \left\{0, \dots, T-2 \right\}$ such that $f_t < 0$, there exists $\bar{t} > t$ such that $\sum_{i=t}^{\bar{t}} f_t \ge 0$.
\end{lemma}
\begin{proof}
We proceed by contradiction. Suppose the claim to be false, then we would have:
\begin{align*}
    \sum_{i=t}^{T-1} f_i < 0.
\end{align*}
However, by manipulating $\sum_{i=t}^{T-1} f_i$, we obtain:
\begin{align*}
    \sum_{i=t}^{T-1} f_i = \sum_{i=t}^{T-1} \gamma^{2i} \Var{R_i} + 2 \sum_{t'=i+1}^{T-2} \gamma^{i+t'} \Cov(R_i, R_{t'}) = \Var\left[\sum_{i=t}^{T-1} \gamma^i R_i \right],
\end{align*}
which is always greater or equal than $0$, thus concluding the proof.
\end{proof}

We then continue by proving the result of Section \ref{sec:emp-problem} that justifies the transformation between optimization problems. However, rather than considering directly the optimization problem we are interested in (i.e., the one defined with $f_t$), we focus on a generalization that consider arbitrary vectors that satisfy the same properties as the one of Lemma \ref{lemma:var-funct-property}.

\begin{lemma}[Optimization of Variance-like functions]\label{lemma:opt-lemma-variance-funct}
Let $c = \left(c_1, \dots, c_k \right)$ with $c_i \in \mathbb{R}$, such that $c_1 < 0$, $\sum_{i=1}^{\bar{k}} c_i \le 0$ for all $\bar{k} < k$, and $\sum_{i=1}^{k} c_i \ge 0$. Let $\Lambda \ge k$ and consider the following optimization problem:
\begin{equation}\label{sys:opt-lemma-variance-funct}
\begin{aligned} 
\min_{\bm{x}} \quad & \sum_{i=1}^{K} \frac{c_t}{x_t} \\
\textrm{s.t.} \quad &  \sum_{t=0}^{T-1} x_t = \Lambda \\
  & x_i \ge x_{i+1}, \quad \forall i \in \{1, \dots, K-1\}  \\
  & x_i \ge 1, \quad \forall i \in \{1, \dots, K\}.
\end{aligned}
\end{equation}
Then, $\bm{\bar{x}} = \left(\frac{\Lambda}{k}, \dots, \frac{\Lambda}{k} \right)$ is an optimal solution of \eqref{sys:opt-lemma-variance-funct}.
\end{lemma}
\begin{proof}
If $\bm{\bar{x}}$ is an optimal solution of \eqref{sys:opt-lemma-variance-funct}, for all $\bm{x} = \left(x_1, \dots, x_k \right)$ that belongs to the feasible region it holds that:
\begin{align*}
    \sum_{i=1}^k \frac{c_i}{\Lambda/k} \le \sum_{i=1}^k \frac{c_i}{x_i},
\end{align*}
which can be rewritten as:
\begin{align}\label{eq:opt-lemma-variance-like-eq1}
    \sum_{i=1}^{k-1} \frac{c_i}{\Lambda/k} \le c_k \left(\frac{1}{x_k} - \frac{1}{\Lambda/k} \right) + \sum_{i=1}^{k-1} \frac{c_i}{x_i} = c_k \frac{\Lambda - kx_k}{x_k \Lambda} + \sum_{i=1}^{k-1} \frac{c_i}{x_i}.
\end{align}
At this point, we notice that $\Lambda \ge k x_k$ for any $\bm{x}$ that belongs to the feasible region. Furthermore, $\sum_{i=1}^k c_i \ge 0$, implies that $c_k \ge -\sum_{i=1}^{k-1} c_i \ge 0$. Therefore, a sufficient for Equation \eqref{eq:opt-lemma-variance-like-eq1} to hold is that:
\begin{align*}
    \sum_{i=1}^{k-1} \frac{c_i}{\Lambda / k} \le \sum_{i=1}^{k-1} \frac{c_i}{x_i} - \sum_{i=1}^{k-1} c_i \left(\frac{1}{x_k} - \frac{1}{\Lambda / k} \right),
\end{align*}
that can be rewritten as:
\begin{align*}
    \sum_{i=1}^{k-1} c_i \left(\frac{1}{x_i} - \frac{1}{x_k} \right) \ge 0,
\end{align*} or equivalently:
\begin{align}\label{eq:opt-lemma-variance-like-eq2}
    \sum_{i=1}^{k-1} c_i \le \sum_{i=1}^{k-1} c_i \frac{x_k}{x_i}.
\end{align}
However, as we shall show, Equation \eqref{eq:opt-lemma-variance-like-eq2} is always satisfied. Indeed, since $\sum_{i=1}^{k-1} c_i \le 0$ and $x_k \le x_{k-1}$ we have that:
\begin{align*}
    \sum_{i=1}^{k-1} c_i \le \sum_{i=1}^{k-1} c_i \frac{x_k}{x_{k-1}} = c_{k-1} \frac{x_k}{x_{k-1}} + \sum_{i=1}^{k-2} c_i \frac{x_k}{x_{k-1}}.
\end{align*}
Moreover, since $\sum_{i=1}^{k-1} c_i \le 0$ and since $x_{k-1} \le x_{k-2}$, 
\begin{align*}
    c_{k-1} \frac{x_k}{x_{k-1}} + \sum_{i=1}^{k-2} c_i \frac{x_k}{x_{k-1}} \le c_{k-1} \frac{x_k}{x_{k-1}} + c_{k-2} \frac{x_k}{x_{k-2}} + \sum_{i=1}^{k-2} c_i \frac{x_k}{x_{k-2}}.
\end{align*}
The properties that $x_i \ge x_{i+1}$ together with the fact that $\sum_{i=1}^{\bar{k}} c_i \le 0$ for any $\bar{k} < k$ allows to iterate the process, thus concluding the proof.
\end{proof}

As one can see, applying multiple times Lemma \ref{lemma:opt-lemma-variance-funct}, to the problem we are considering, we obtain a transformed problem that is convex, since the objective function will be composed of summation of convex functions. We will provide additional details on this point later on. We now continue by studying the properties of optimization problems whose objective function satisfies the condition of Lemma \ref{lemma:opt-lemma-variance-funct}. More specifically, the following Lemma allows us to quantify the difference in the optimal solution when changing the budget constraint.

\begin{lemma}[Budget sensitivity analysis]\label{lemma:restricted-budget-opt}
Let $c_t \in \mathbb{R}$ for each $t \in \{0, \dots, T-1 \}$. Define $\mathcal{Y} = \left\{i \in \left\{0, \dots, T-1 \right\} : c_i < 0 \right\}$. Let $y \in \mathcal{Y}$, and define $q(y)$ as the smallest integer in $\left\{y+1, \dots, T-1 \right\}$ such that $\sum_{i=y}^{q(y)} c_i \ge 0$. Suppose that $q(y)$ is well-defined for any $y \in \mathcal{Y}$.  
%\am{Hai mostrato che $q(y)$ è sempre definite?}
%\rp{In questo Lemma lo assumo vero. Per quanto riguarda il problema empirico invece viene fatto in Lemma \ref{lemma:high-prob-convex-problem}. Per il problema originale, invece, ho aggiunto ora il Lemma \ref{lemma:var-funct-property}.}

Consider the following optimization problems:
\begin{equation}\label{sys:prob-1}
\begin{aligned} 
\min_{\bm{x}} \quad & \sum_{t=0}^{T-1} \frac{c_t}{x_t}  \\
\textrm{s.t.} \quad & \sum_{t=0}^{T-1} x_t = \Lambda \\
  & x_t \ge x_{t+1}, \quad \forall t \in \{0, \dots, T-2\}  \\
  & x_t \ge 0, \quad \forall t \in \{0, \dots, T-1\} \\
  & x_y = x_{y+1} = \dots =x_{q(y)}, \quad \forall y \in \mathcal{Y},
\end{aligned}
\end{equation}
and,
\begin{equation}\label{sys:prob-2}
\begin{aligned} 
\min_{\bm{x}} \quad & \sum_{t=0}^{T-1} \frac{c_t}{x_t}  \\
\textrm{s.t.} \quad & \sum_{t=0}^{T-1} x_t = \Lambda' \\
  & x_t \ge x_{t+1}, \quad \forall t \in \{0, \dots, T-2\}  \\
  & x_t \ge 0, \quad \forall t \in \{0, \dots, T-1\} \\
  & x_y = x_{y+1} = \dots = x_{q(y)}, \quad \forall y \in \mathcal{Y}, \\
\end{aligned}
\end{equation}
where $\Lambda, \Lambda' \in \mathbb{R}$ such that $\Lambda \ge T$ and $\Lambda' \ge T$. Define $\alpha = \frac{\Lambda'}{\Lambda}$ and consider $\bm{x}^*$ an optimal solution of \eqref{sys:prob-1}. Then, $\alpha \bm{x}^*$ is an optimal solution of \eqref{sys:prob-2}.
\end{lemma}
\begin{proof}
First of all, it is important to notice that both problems takes finite and positive value. This directly follow from the equality constraints, together with the fact that $q(y)$ is well-defined for any $y \in \mathcal{Y}$. 

We now continue in proving the claim. Proceed by contradiction and suppose that $\alpha \bm{x}^*$ is not an optimal solution of \eqref{sys:prob-2}, and let $\bm{\bar{x}}$ be an optimal solution of \eqref{sys:prob-2}.

At this point, first of all, we notice that $\alpha \bm{x}^*$ is a feasible solution of \eqref{sys:prob-2}. Indeed, we have that $\alpha x_{t}^* \ge 0 $, $\alpha x_{t}^* \ge \alpha x_{t+1}^*$, $\alpha \sum_{t=0}^{T-1} x_{t}^* = \alpha \Lambda = \frac{\Lambda'}{\Lambda} \Lambda = \Lambda'$, and for all $y \in \mathcal{Y}$, $\alpha x^*_y = \alpha x^*_{y+1} = \dots = \alpha x^*_{q(y)}$.  

Therefore, we can write:
\begin{align*}
\sum_{t=0}^{T-1} \frac{c_t}{\bar{x}_t^*} < \sum_{t=0}^{T-1}  \frac{c_t}{\alpha x_{t}^*} = \frac{1}{\alpha} \sum_{t=0}^{T-1} \frac{c_t}{x_{t}^*}.
\end{align*}
From which it follows that:
\begin{align*}
\sum_{t=0}^{T-1} \frac{c_t}{x_{t}^*} > \sum_{t=0}^{T-1} \frac{c_t}{\bar{x}_t^* / \alpha}.
\end{align*}
However, for similar reasoning w.r.t. to the ones presented above, $\left( \bar{x}_1^* / \alpha, \dots, \bar{x}_{T-1}^* / \alpha \right)$ is a feasible solution for \eqref{sys:prob-1}, from which it follows that $\bm{{x}}^*$ would not be optimal, which is impossible.

\end{proof}

The following result, instead, is a technical Lemma that will be used to analyze the error that \rido accumulates in each optimization round.

\begin{lemma}[Technical lemma]\label{lemma:tech-lemma-reduction-improved}
Consider a sequence of $K \in \mathbb{N}$ elements $(a_1, \dots, a_K)$ such that $a_i \in \mathbb{R}$ and $a_i > 0$ for all $i \in [K]$. Then:
\begin{align}\label{eq:tech-lemma-reduction-claim-improved}
    \frac{1}{\sum_{i=1}^K a_i } \le \frac{1}{K^2} \sum_{i=1}^K \frac{1}{a_i}.
\end{align}
\end{lemma}
\begin{proof}
We begin with some notation. Consider $K \in \mathbb{N}$ such that $K > 1$, we denote with $\mathcal{V}_K$ the subset of entry-wise strictly positive vectors of $\mathbb{R}^K$, namely:
\begin{align*}
    \mathcal{V}_K = \left\{ (a_1, \dots, a_K) \in \mathbb{R}^K | a_i > 0 \text{ for all } i \in [K] \right\}.
\end{align*}

We now proceed by induction on $K$. 

Consider $K=1$ and $\bm{v} = (a_1) \in \mathcal{V}_1$. In this case, Equation \eqref{eq:tech-lemma-reduction-claim-improved} holds for all $\bm{v} \in \mathcal{V}_1$ since it reduces to:
\begin{align*}
    \frac{1}{a_1} \le \frac{1}{a_1}.
\end{align*}

At this point, suppose that:
\begin{align*}
    \frac{1}{\sum_{i=1}^{K} a_i} \le \frac{1}{K^2}\sum_{i=1}^{K} \frac{1}{a_i},
\end{align*}
holds for $K$ and for all vectors $\bm{v}_K \in \mathcal{V}_K$, and consider:
\begin{align*}
    \frac{1}{\sum_{i=1}^{K+1} a_i} \le \frac{1}{(K+1)^2}\sum_{i=1}^{K+1} \frac{1}{a_i},
\end{align*}
for any vector $\bm{v}_{K+1} = (a_1, \dots a_{K+1}) \in \mathcal{V}_{K+1}$.  At this point, notice that, for all $\bm{v}_{K+1} \in \mathcal{V}_{K+1}$ the vector $\bm{v}_{K, -i}$ that is obtained from $\bm{v}_{K+1}$ by removing the $i$-th component belongs to $\mathcal{V}_K$. At this point, focus on:
\begin{align*}
    \frac{1}{(K+1)^2}\sum_{i=1}^{K+1} \frac{1}{a_i} = \frac{1}{(K+1)^2}\left(\sum_{i=1}^K \frac{1}{a_i} + \frac{1}{a_{k+1}} \right).
\end{align*}
Thanks to the inductive hypothesis and some algebraic manipulations, we have that:
\begin{align*}
\frac{1}{(K+1)^2}\left(\sum_{i=1}^K \frac{1}{a_i} + \frac{1}{a_{k+1}} \right) & = \frac{K^2}{K^2 (K+1)^2} \sum_{i=1}^K \frac{1}{a_i} + \frac{1}{(K+1)^2} \frac{1}{a_{k+1}} \\ & \ge \frac{K^2}{(K+1)^2} \left( \frac{1}{\sum_{i=1}^{K} a_i} \right)+ \frac{1}{(K+1)^2 a_{K+1}} \\ & = \frac{K^2}{(K+1)^2} \left( \frac{1}{\sum_{i=1}^{K} a_i} + \frac{1}{a_{K+1}K^2} \right).
\end{align*}
At this point, we need to show that:
\begin{align*}
    \frac{K^2}{(K+1)^2} \left( \frac{1}{\sum_{i=1}^{K} a_i} + \frac{1}{a_{K+1}K^2} \right) \ge \frac{1}{\sum_{i=1}^{K} a_i + a_{K+1}},
\end{align*}
holds. Set, for the sake of exposition $c = \sum_{i=1}^{K} a_i$ and $d = a_{K+1}$. Then, we can rewrite the previous inequality as:
\begin{align*}
\frac{K^2}{(K+1)^2} \left(\frac{1}{c} + \frac{1}{K^2 d} \right) \ge \frac{1}{c+d}.
\end{align*}
Rearranging the terms we obtain:
\begin{align*}
\frac{K^2}{(K+1)^2} \left(\frac{K^2 d + c}{cdK^2} \right) \ge \frac{1}{c+d}.
\end{align*}
Which, in turns, lead to:
\begin{align*}
    K^2 (K^2d + c) (c+d) \ge (K+1)^2 cdK^2.
\end{align*}
Multiplying each term and dividing by $K^2$ leads to:
\begin{align*}
    d^2 K^2 - cdK + c^2 \ge 0,
\end{align*}
which holds for any value of $K > 0$, and $d, c > 0$, thus concluding the proof.
\end{proof}

Finally, the following Lemma will be used to take into account the rounding effect that comes from solving a continuous relaxation rather than an integer optimization problem.

\begin{lemma}[Rounding effect error]\label{lemma:rounding-effect-error}
Consider a generic $T$-dimensional vector $\bm{{n}} = \left(n_0, \dots n_{T-1} \right)$ such that $n_i \ge 1$ for all $i \in \{0, \dots, T-1 \}$. Let $q = \sum_{t=0}^{T-1} n_{t}$, and define $k = q - \sum_{t=0}^{T-1} \lfloor n_t \rfloor$. Consider the vector $\bm{\bar{n}} = \left(\bar{n}_0, \dots \bar{n}_{T-1} \right)$ such that:
\begin{align*}
    \bar{n}_t = \lfloor n_t \rfloor + \bm{1}\left\{ t < k \right\}.
\end{align*}
Define $g(\bm{n}) = \sum_{t=0}^{T-1} \frac{c_t}{n_t}$ for some vector $\bm{c} = \left(c_0, \dots, c_{T-1} \right)$ with $c_{t} \in \mathbb{R}$. Then, the following holds:
\begin{align}\label{eq:rounding-effect-eq1}
    \sum_{t: c_t \ge 0} \frac{c_t}{\bar{n}_t} \le 2 \sum_{t: c_t \ge 0} \frac{c_t}{n_t},
\end{align}
\begin{align}\label{eq:rounding-effect-eq2}
    \sum_{t: c_t \le 0} \frac{c_t}{\bar{n}_t} \le \frac{1}{2} \sum_{t: c_t \le 0} \frac{c_t}{n_t}.
\end{align}
\end{lemma}
\begin{proof}
We begin by proving Equation \eqref{eq:rounding-effect-eq1}. First of all, let us notice that:
\begin{align}\label{eq:rounding-effect-eq3}
    \sum_{t: c_t \ge 0} \frac{c_t}{n_t} \ge \sum_{t: c_t \ge 0} \frac{c_t}{\bar{n}_t + 1} \ge \sum_{t: c_t \ge 0} \frac{c_t}{2\bar{n}_t},
\end{align}
where in the first inequality we have used $c_t \ge 0$ together with $|n_t - \bar{n}_t| \le 1$, while in the second one we have used $c_t \ge 0$ together with $\bar{n}_t \ge 1$. Equation \eqref{eq:rounding-effect-eq1} directly follows from Equation \eqref{eq:rounding-effect-eq3}.

We continue by proving Equation \eqref{eq:rounding-effect-eq2}. Similar to Equation \eqref{eq:rounding-effect-eq3}, it is possible to obtain:
\begin{align}\label{eq:rounding-effect-eq4}
    \sum_{t: c_t \le 0} \frac{c_t}{\bar{n}_t} \le \sum_{t: c_t \le 0} \frac{c_t}{n_t + 1} \le \sum_{t: c_t \le 0} \frac{c_t}{2 n_t} = \frac{1}{2} \sum_{t: c_t \le 0} \frac{c_t}{n_t},
\end{align}
where in the first step we have used $c_t \le 0$ together with $c_t \le 0$, while in the second one we have used $c_t \le 0$ together with $n_t \ge 1$.
\end{proof}

\subsubsection{\rido analysis}
We begin with some concentration inequalities. We report for completeness the result (Theorem $10$) of \citet{maurer2009empirical} that we use to construct confidence intervals around the standard deviation.

\begin{lemma}[Standard deviation confidence intervals]\label{lemma:emp-bernstein}
Let $n \ge 2$ and consider $X_1, \dots, X_n$ be i.i.d. random variables with values in $[0, 1]$. Define:
\begin{align*}
    \hat{\sigma} = \sqrt{\frac{1}{n(n-1)} \sum_{i<j} (X_i - X_j)^2}.
\end{align*}
Then, for $\delta \in (0,1)$, with probability at least $1-\delta$ we have that:
\begin{align*}
    |\hat{\sigma} - \sigma| \le \sqrt{\frac{2 \ln(1/\delta)}{n-1}},
\end{align*}
where $\sigma  = \E \hat{\sigma}$. 
\end{lemma}

We then continue with similar results for the estimation of the covariances between random variables.

\begin{lemma}[Covariance confidence intervals]\label{lemma:ci-cov-iter}
Consider $(X_1, Y_1), \dots (X_n, Y_n)$ i.i.d. random variables with values in $[0,1]$ sampled from the joint distribution $f_{X,Y}$. Moreover, let $X_{n+1}, \dots, X_{n+k}$ be $k$ i.i.d. random variables with values in $[0,1]$ sampled from distribution $f_X = \mathbb{E}_{Y}\left[ f_{X,Y}\right]$. Define, for all $i \in [n]$, $Z_i = X_i Y_i$, and let $\hat{z} = \frac{1}{n} \sum_{i=1}^n Z_i$, $\hat{x} = \frac{1}{n+k} \sum_{i=1}^{n+k} X_i$ and $\hat{y} = \frac{1}{n} \sum_{i=1}^n Y_i$.
Then, for $\delta \in (0,1)$, we have that:
\begin{align*}
|\E \hat{z} - \E \hat{x} \E \hat{y} - (\hat{z} - \hat{x} \hat{y})| \le 3\sqrt{\frac{2 \log(6/\delta)}{n}}.
\end{align*}

\end{lemma}

\begin{proof}
By Hoeffding Inequality \citet{boucheron2003concentration}, we have that, for some confidence level $\delta'$, the following holds with probability at least $1-\delta'$:
\begin{align*}
    |\hat{z} - \E \hat{z}| \le \sqrt{\frac{2 \log(2/\delta')}{n}},
\end{align*}
and, similarly for $\hat{x}$ and $\hat{y}$. Therefore, by Boole's inequality, it follows that, with probability at least $1-\delta$, we have that: 
\begin{align}\label{eq:cov-ci-eq1}
    |\hat{z} - \E \hat{z}| \le \sqrt{\frac{2 \log(6/\delta)}{n}},
\end{align}
and, similarly, for $\hat{x}$ and $\hat{y}$. \footnote{For $\hat{x}$ the confidence intervals holds with $\sqrt{\frac{2 \log(6/\delta)}{n+k}}$, which is possibly smaller since $n \le n+k$.}

Therefore, with probability at least $1-\delta$ we have that:
\begin{align*}
    |\E \hat{z} - \E \hat{x} \E \hat{y} - (\hat{z} - \hat{x} \hat{y})| & \le |\E \hat{z} - \hat{z}| + |\E \hat{x} \E \hat{y} - \hat{x} \hat{y}| \\ & \le \sqrt{\frac{2 \log(6/\delta)}{n}} + |\E \hat{x} \E \hat{y}  - \hat{y} \E \hat{x} + \hat{y} \E x - \hat{x} \hat{y}| \\ & \le \sqrt{\frac{2 \log(6/\delta)}{n}} + |\E \hat{x} (\E \hat{y} - \hat{y})| + | \hat{y} (\E \hat{x} - \hat{x}) | \\ & \le 2\sqrt{\frac{2 \log(6/\delta)}{n}} + |\hat{y}| \sqrt{\frac{2 \log(6/\delta)}{n}} \\ & \le 3\sqrt{\frac{2 \log(6/\delta)}{n}}.
\end{align*}
where we combined Equation \eqref{eq:cov-ci-eq1} together with triangular inequalities.
\end{proof}

At this point, before diving into the presentation of the good event under which we will conduct our analysis, we provide a formal definition of our estimators. Consider a generic dataset of trajectories of different lenght. Define, for each $t \in \{0, \dots, T-1\}$:
\begin{align}\label{def:emp-rew-var}
    \sqrt{\Empvar\left( R_t \right)} = \sqrt{\frac{1}{n_t(n_t-1)} \sum_{1 \le i < j \le n} \left(R_t^{(i)} - R_t^{(j)}\right)^2},
\end{align}
where $R_t^{(i)}$ denotes the reward gathered at step $t$ in some trajectory whose length is at least $t+1$. Moreover, for $t, t'$ such that $t < t'$, define:
\begin{align}\label{def:emp-rew-cov}
    \Empcov(R_t, R_t') = \frac{1}{n_{t'}}\sum_{i=1}^{n_{t'}} R_t^{(i)} R_{t'}^{(i)} - \left(\frac{1}{n_t}\sum_{i=1}^{n_t} R_t^{(i)}\right)  \left(\frac{1}{n_{t'}}\sum_{i=1}^{n_{t'}} R_{t'}^{(i)}\right).
\end{align}

\begin{lemma}[Good event]\label{lemma:iterative-good-event}
The following conditions holds for all phases of \rido, with probability at least $1-\delta$:
\begin{align}\label{def:var-ci}
\Big|\sqrt{\Var\left(R_t\right)} - \sqrt{\Empvar_i\left( R_t \right)} \Big| \le \sqrt{\frac{2 \log\left(\frac{6(T+T^2)\Lambda K}{\delta} \right)}{\sum_{j=0}^i N_{t,i}}} = \textrm{C}^\sigma_{i,t}.
\end{align}
and:
\begin{align}\label{def:cov-ci}
\Big|\Cov\left(R_t,R_{t'}\right) - \Empcov_i\left(R_t, R_{t'}\right)\Big| \le 3 \sqrt{\frac{2\log\left(\frac{6(T+T^2)\Lambda K}{\delta} \right)}{\sum_{j=0}^i N_{t',i}}} = \textrm{C}^{c}_{i,t,t'}.
\end{align}
\end{lemma}
\begin{proof}
The proof follows by combining Lemma \ref{lemma:ci-cov-iter} and Lemma \ref{lemma:emp-bernstein}, and by taking the union bound over the different time steps, optimization rounds, and possible ways in which the budget can be spent.
\end{proof}

At this point, first we show that, with high probability, the objective function of the empirical optimization problem \eqref{sys:emp-opt-prob} satisfies the same property of the objective function of the original optimization problem \eqref{sys:opt-prob}, i.e., Lemma \ref{lemma:var-funct-property}. Consequently, it holds that the procedure described in the main text in Section \ref{sec:emp-problem} leads to a transformed convex optimization problem that preserves the optimal solution. For this reason, in the rest of this section, under the good event of Lemma \ref{lemma:iterative-good-event}, we assume that \rido has actually access to an optimal solution of the continuous relaxation of \eqref{sys:emp-opt-prob}, which can be obtained in a computational efficient way by transforming the optimization problem.

\begin{lemma}[High probability property of the empirical problem]\label{lemma:high-prob-convex-problem}
Let $\beta = \frac{6(T+T^2)\Lambda K}{\delta}$ and consider a generic phase $i$ of Algorithm \ref{alg:rido}. Define:
\begin{align*}
    \hat{f}_{t,i} = \gamma^{2t} \left( \sqrt{\Empvar_i \left( {R}_t \right)} + \textrm{C}^\sigma_{i,t} \right)^2 + 2 \sum_{t'=t+1}^{T-1} \gamma^{t+t'}  \left( \Empcov_i(R_t, R_{t'}) + \textrm{C}^{c}_{i,t,t'} \right).
\end{align*}
Suppose that $\hat{f}_{t,i} < 0$. Then, with probability at least $1-\delta$, for any $t \in \left\{0, \dots, T-2 \right\}$ there exists $\bar{t} > t$ such that $\sum_{j=t}^{\bar{t}} \hat{f}_{j,i} \ge 0$ holds. 
\end{lemma}

\begin{proof}
We proceed by contradiction. Suppose that $\hat{f}_{t,i} < 0$ and $\sum_{j=t}^{\bar{t}} \hat{f}_{j,i} < 0$ for all $\bar{t} > t$, and, thus, also for $\bar{t} = T-1$. Due to Lemma \ref{lemma:iterative-good-event}, we have that:
\begin{align*}
    \sum_{j=t}^{T-1} \hat{f}_{j,i} \ge \sum_{j=t}^{T-1} \gamma^{2j} \Var\left(R_j\right) + 2 \sum_{t'=j+1}^{T-1} \gamma^{i+t'}\Cov\left(R_j, R_{t'}\right) = \Var\left( \sum_{j=t}^{T-1} \gamma^j R_j \right), 
\end{align*}
which, however, is always greater or equal than $0$, thus leading to a contradiction and concluding the proof.
\end{proof}

To analyze the performance of \rido, we will focus the following quantity:
\begin{align}\label{eq:regret}
\sum_{t=0}^{T-1} \frac{f_t}{N_t} - \Var\left[ \hat{J}_{\bm{{n}^*}}(\pi) \right] = \sum_{t=0}^{T-1} \frac{f_t}{\sum_{i=0}^{K-1} N_{t,i}} - \sum_{t=0}^{T-1} \frac{f_t}{n_t^*}.
\end{align}
In the following Lemma we decompose Equation \eqref{eq:regret} into several terms.

\begin{lemma}[Error decomposition]\label{lemma:regret-decomposition}
Let $f_t = \gamma^{2t} \Var\left( R_t \right) + 2 \sum_{t'=t+1}^{T-1} \gamma^{t+t'} \Cov\left(R_t, R_{t'} \right)$. Let $y \in \mathcal{Y}$, and define $q(y)$ as the smallest integer in $\left\{y+1, \dots, T-1 \right\}$ such that $\sum_{i=y}^{q(y)} f_i \ge 0$. Equation \eqref{eq:regret} can be upper bounded by:
\begin{equation}\label{eq:regre-decomposition-eq1}
\begin{aligned}
    \frac{2}{K^2} & \sum_{i=1}^{K} \sum_{t=0}^{T-1} f_t \left( \frac{1}{\bar{N}_{t,i}} - \frac{1}{\tilde{x}^*_t + 1} \right) + \frac{2}{K^2} \sum_{i=1}^{K} \sum_{t: f_t < 0} \frac{|f_t|}{\bar{N}_{t,i}}  - \sum_{t: f_t < 0} \frac{|f_t|}{\sum_{i=1}^{k} N_{t,i}} \\
   & + \frac{2}{K} \sum_{t=0}^{T-1} \frac{f_t}{\tilde{x}^*_{t} + 1} - \sum_{t=0}^{T-1} \frac{f_t}{n_t^*}.
\end{aligned}
\end{equation}
where $\bm{\bar{N}}_i$ is the optimal solution of the continuous relaxation \eqref{sys:emp-opt-prob}, $\bm{N}_i$ is the DCS obtained from rounding $\bm{\bar{N}}_i$, and $\tilde{x}^*$ is the optimal solution of the following optimization problem:
\begin{equation}
\begin{aligned} 
\min_{\bm{x}} \quad & \sum_{i=1}^{K} \frac{f_t}{x_t} \\
\textrm{s.t.} \quad &  \sum_{t=0}^{T-1} x_t = b - T \\
  & x_i \ge x_{i+1}, \quad \forall i \in \{1, \dots, K-1\}  \\
  & x_i \ge 0, \quad \forall i \in \{1, \dots, K\} \\
  & x_y = x_{y+1} = \dots = x_{q(y)}, \quad \forall y \in \mathcal{Y}.
\end{aligned}
\end{equation}
\end{lemma}
\begin{proof}
Focus on Equation \eqref{eq:regret}:
\begin{align*}
    \mathcal{R} & \coloneqq \sum_{t=0}^{T-1} \frac{f_t}{\sum_{i=1}^{k} N_{t,i}} - \sum_{t=0}^{T-1} \frac{f_t}{n_t^*}  \\ & = \sum_{t: f_t \ge 0} \frac{f_t}{\sum_{i=1}^{k} N_{t,i}} + \sum_{t: f_t < 0} \frac{f_t}{\sum_{i=1}^{k} N_{t,i}} - \sum_{t=0}^{T-1} \frac{f_t}{n_t^*} .
\end{align*}
Using Lemma \ref{lemma:tech-lemma-reduction-improved}, we obtain that:
\begin{align*}
    \mathcal{R} & \le \frac{1}{K^2} \sum_{i=1}^K \sum_{t: f_t \ge 0} \frac{f_t}{N_{t,i}} + \sum_{t: f_t < 0} \frac{f_t}{\sum_{i=1}^{k} N_{t,i}} - \sum_{t=0}^{T-1} \frac{f_t}{n_t^*} \\ & \le  \frac{2}{K^2} \sum_{i=1}^K \sum_{t: f_t \ge 0} \frac{f_t}{\bar{N}_{t,i}} + \sum_{t: f_t < 0} \frac{f_t}{\sum_{i=1}^{k} N_{t,i}} - \sum_{t=0}^{T-1} \frac{f_t}{n_t^*} \\ & =  \frac{2}{K^2} \sum_{i=1}^K \sum_{t=0}^{T-1} \frac{f_t}{\bar{N}_{t,i}} - \frac{2}{K^2} \sum_{i=1}^{K} \sum_{t: f_t < 0} \frac{f_t}{\bar{N}_{t,i}} + \sum_{t: f_t < 0} \frac{f_t}{\sum_{i=1}^{k} N_{t,i}} - \sum_{t=0}^{T-1} \frac{f_t}{n_t^*}, 
\end{align*}
where (i) the first step is due to Lemma \ref{lemma:tech-lemma-reduction-improved}, (ii) the second one follows from using Lemma \ref{lemma:rounding-effect-error}, and (iii) the third one from adding and subtracting $\frac{2}{K^2} \sum_{i=1}^K \sum_{t: f_t < 0} \frac{f_t}{\bar{N}_{t,i}}$. The proof follows by adding and subtracting:
\begin{align*}
    \frac{2}{K^2} \sum_{i=1}^K \sum_{t=0}^{T-1} \frac{f_t}{\tilde{x}^*_t + 1} = \frac{2}{K} \sum_{t=0}^{T-1} \frac{f_t}{\tilde{x}^*_t + 1}.
\end{align*}
\end{proof}

We now upper bound each individual term of this error decomposition. We start from the first one, that is:
\begin{align*}
\frac{2}{K^2} \sum_{i=1}^{K} \sum_{t=0}^{T-1} f_t \left( \frac{1}{\bar{n}_{t,i}} - \frac{1}{\tilde{x}^*_t + 1} \right).
\end{align*}

\begin{lemma}[Cumulative error]\label{lemma:iterative-sampling-error}
Let $\beta = \frac{6(T+T^2)\Lambda K }{\delta}$. Let $y \in \mathcal{Y}$, and define $q(y)$ as the smallest integer in $\left\{y+1, \dots, T-1 \right\}$ such that $\sum_{i=y}^{q(y)} f_i \ge 0$.
Let $\bm{\tilde{x}}^*$ be the solution of the following optimization problem:
\begin{equation}
\begin{aligned} 
\min_{\bm{x}} \quad & \sum_{i=1}^{K} \frac{f_t}{x_t} \\
\textrm{s.t.} \quad &  \sum_{t=0}^{T-1} x_t = b - T \\
  & x_i \ge x_{i+1}, \quad \forall i \in \{1, \dots, K-1\}  \\
  & x_i \ge 0, \quad \forall i \in \{1, \dots, K\} \\
  & x_{y} = x_{y+1} = \dots = x_{q(y)}, \quad \forall y \in \mathcal{Y},
\end{aligned}
\end{equation}
and let $\bm{\bar{N}}_i$ be the solution of the continuous relaxation of \eqref{sys:emp-opt-prob} during phase $i$. Then, with probability at least $1-\delta$, the following holds:
\begin{align}
    \frac{2}{K^2} \sum_{i=1}^K \sum_{t=0}^{T-1} f_t \left(\frac{1}{\bar{N}_{t,i}} - \frac{1}{\tilde{x}^*_t + 1} \right) \le \frac{192}{K^{\frac{3}{2}}} \log\left( \frac{6(T+T^2)\Lambda K}{\delta} \right) \left(\sum_{t=0}^{T-1} \gamma^t \right)^2 
\end{align}
\end{lemma}
\begin{proof}
The proof is split into $3$ parts. First we focus on
\begin{align*}
    \frac{2}{K^2} \sum_{i=1}^K \sum_{t=0}^{T-1} f_t \left(\frac{1}{\bar{N}_{t,i}} - \frac{1}{\tilde{x}^*_t + 1} \right)
\end{align*}
for a generic phase $i > 1$, then for $i=1$, and finally we combine everything together. 

\paragraph{Analysis of a generic phase $i > 1$}
Fix a phase $i > 1$, and analyze: 
\begin{align}\label{eq:iterative-error-eq1}
    \sum_{t=0}^{T-1} f_t \left(\frac{1}{\bar{N}_{t,i}} - \frac{1}{\tilde{x}^*_t + 1} \right)
\end{align}
Focus on $\sum_{t=0}^{T-1} \frac{f_t}{\tilde{x}_t + 1}$, and let $\bm{\tilde{g}} = \left(\tilde{g}_0, \dots, \tilde{g}_{T-1} \right)$ be the solution to the following optimization problem:
\begin{equation}
\begin{aligned} 
\min_{\bm{g}} \quad & \sum_{t=0}^{T-1} \frac{f_t}{g_t}  \\
\textrm{s.t.} \quad &  \sum_{t=0}^{T-1} g_t = b \\
  & g_t \ge g_{t+1}, \quad \forall t \in \{0, \dots, T-2\}  \\
  & g_t \ge 1, \quad \forall t \in \{0, \dots, T-1\}.
\end{aligned}
\end{equation}
We can observe that: \footnote{This step follows by considering the optimization problem that defines $\bm{\tilde{g}}$. With a change of variable $g_t = x_t + 1$, we have that $\tilde{x}^*_t + 1$ is a feasible solution of the same optimization problem. Moreover, due to Lemma \ref{lemma:opt-lemma-variance-funct}, we can neglect the constraints on $y$.}
\begin{align*}
    \sum_{t=0}^{T-1} \frac{f_t}{\tilde{g}_t} \le \sum_{t=0}^{T-1} \frac{f_t}{\tilde{x}^*_t + 1}.
\end{align*}
Plugging this result into Equation \eqref{eq:iterative-error-eq1} we obtain:
\begin{align}\label{eq:iterative-error-eq2}
    \sum_{t=0}^{T-1} f_t \left(\frac{1}{\bar{N}_{t,i}} - \frac{1}{\tilde{x}^*_t + 1} \right) \le \sum_{t=0}^{T-1} f_t \left(\frac{1}{\bar{N}_{t,i}} - \frac{1}{\tilde{g}^*_t} \right).
\end{align}
With probability at least $1-\delta$ (i.e., Lemma \ref{lemma:iterative-good-event}), we can upper bound Equation \eqref{eq:iterative-error-eq2} by:
\begin{align*}
\sum_{t=0}^{T-1} \frac{\gamma^{2t} \left( \sqrt{\Empvar_i\left( R_t \right)}  + \textrm{C}^\sigma_{i,t} \right)^2}{\bar{N}_{t,i}} + 2 \sum_{t=0}^{T-2} \sum_{t'=t+1}^{T-1} \frac{\gamma^{t+t'}\left( \Empcov(R_t, R_{t'}) + \textrm{C}^c_{i,t,t'} \right)}{\bar{N}_{t,i}} - \sum_{t=0}^{T-1} \frac{f_t}{\tilde{g}_t},
\end{align*}
Moreover, since $\bm{\tilde{g}}$ is a feasible solution of the continuous relaxation of \eqref{sys:emp-opt-prob}, and since $\bar{N}_{t,i}$ is the minimizer of the continuous relaxation of \eqref{sys:emp-opt-prob} at phase $i$, we can further bound the previous equation with:
\begin{align*}
\sum_{t=0}^{T-1} \frac{\gamma^{2t} \left( \sqrt{\Empvar_i\left( R_t \right)}  + \textrm{C}^\sigma_{i,t} \right)^2}{\tilde{g}_t} + 2 \sum_{t=0}^{T-2} \sum_{t'=t+1}^{T-1} \frac{\gamma^{t+t'}\left( \Empcov(R_t, R_{t'}) + \textrm{C}^c_{i,t,t'} \right)}{\tilde{g}_t} - \sum_{t=0}^{T-1} \frac{f_t}{\tilde{g}_t}.
\end{align*}
Applying Lemma \ref{lemma:iterative-good-event}, we can further bound the previous equation by:
\begin{align}\label{eq:iterative-error-new}
\sum_{t=0}^{T-1} \frac{\gamma^{2t} \left( \sqrt{\Var\left[ R_t \right]}  + 2\textrm{C}^\sigma_{i,t} \right)^2}{\tilde{g}_t} + 2 \sum_{t=0}^{T-2} \sum_{t'=t+1}^{T-1} \frac{\gamma^{t+t'}\left( \Cov(R_t, R_{t'}) + 2\textrm{C}^c_{i,t,t'} \right)}{\tilde{g}_t} - \sum_{t=0}^{T-1} \frac{f_t}{\tilde{g}_t},
\end{align}
Focus on:
\begin{align}\label{eq:iterative-error-eq3}
    \sum_{t=0}^{T-1} \frac{\gamma^{2t} \left( \sqrt{\Var\left[ R_t \right]}  + 2\textrm{C}^\sigma_{i,t} \right)^2}{\tilde{g}_t} + 2 \sum_{t=0}^{T-2} \sum_{t'=t+1}^{T-1} \frac{\gamma^{t+t'}\left( \Cov(R_t, R_{t'}) + 2\textrm{C}^{c}_{i,t,t'} \right)}{\tilde{g}_t}.
\end{align}
This equation can be decomposed into:
\begin{align}\label{iterative-error-eq4}
    \sum_{t=0}^{T-1} \frac{\gamma^{2t}  4\sqrt{\Var\left[ R_t \right]}\textrm{C}^\sigma_{i,t}  + 4 \gamma^{2t} \left(\textrm{C}^\sigma_{i,t}\right)^2 }{\tilde{g}_t} + 4 \sum_{t=0}^{T-2} \sum_{t'=t+1}^{T-1} \frac{\gamma^{t+t'} \textrm{C}^{c}_{i,t,t'} }{\tilde{g}_t}
\end{align}
and,
\begin{align}\label{eq:iterative-error-eq5}
    \sum_{t=0}^{T-1} \frac{f_t}{\tilde{g}_t}.
\end{align}
Using this decomposition within Equation \eqref{eq:iterative-error-new}, we have:
\begin{align}\label{eq:iterative-error-eq6}
    \sum_{t=0}^{T-1} \frac{\gamma^{2t}  4\sqrt{\Var\left[ R_t \right]}\textrm{C}^\sigma_{i,t}  + 4 \gamma^{2t} \left(\textrm{C}^\sigma_{i,t}\right)^2 }{\tilde{g}_t} + 4 \sum_{t=0}^{T-2} \sum_{t'=t+1}^{T-1} \frac{\gamma^{t+t'} \textrm{C}^{c}_{i,t,t'} }{\tilde{g}_t}
\end{align}
We now upper bound each term in Equation \eqref{eq:iterative-error-eq6}. For brevity, we define $h_{i,t} \coloneqq \sum_{j=0}^{i-1} N_{t,j}$. Then, we have that:
\begin{align*}
    \sum_{t=0}^{T-1} \frac{4\gamma^{2t} \sqrt{\Var\left[ R_t \right]}\textrm{C}^\sigma_{i,t}}{\tilde{g}_t} & \le \sum_{t=0}^{T-1} \frac{4\gamma^{2t}}{\tilde{g}_t} \sqrt{\frac{2\log\left( \frac{2(T+T^2)\Lambda K}{\delta} \right)}{h_{i-1, t}}} \\ & \le 8 \sqrt{\log\left( \frac{2(T+T^2)\Lambda K}{\delta} \right)} \sum_{t=0}^{T-1} \frac{\gamma^{2t}}{\tilde{g}_t \sqrt{i-1}} \\ & \le 16 \sqrt{\log\left( \frac{2(T+T^2)\Lambda K}{\delta} \right)} \sum_{t=0}^{T-1} \frac{\gamma^{2t}}{\sqrt{i}}.
\end{align*}
where (i) the first step follows from the definition of the confidence intervals, together with the fact that rewards are bounded in $[0,1]$, (ii)  the second one from the fact $h_{i-1,t} = \sum_{j=1}^{i-1} N_{t,j} \ge i - 1$, and (iii) the third one from $\sqrt{i} \le 2 \sqrt{i-1}$.
Similary, we obtain that:
\begin{align*}
    \sum_{t=0}^{T-1} \frac{4\gamma^{2t} \left(\textrm{C}^{\sigma}_{i,t} \right)^2}{\tilde{g}_t} & \le \sum_{t=0}^{T-1} \frac{4\gamma^{2t}}{\tilde{g}_t} \frac{2 \log\left( \frac{2(T+T^2)\Lambda K}{\delta} \right)}{h_{i-1,t}} \\ & \le 8 \log\left( \frac{2(T+T^2)\Lambda K}{\delta} \right) \sum_{t=0}^{T-1} \frac{\gamma^{2t}}{\tilde{g}_t (i-1)} \\ & \le 16 \log\left( \frac{2(T+T^2)\Lambda K}{\delta} \right) \sum_{t=0}^{T-1} \frac{\gamma^{2t}}{\sqrt{i}}.
\end{align*}
Finally, we have that:
\begin{align*}
    4 \sum_{t=0}^{T-2} \sum_{t'=t+1}^{T-1} \frac{\gamma^{t+t'} \textrm{C}^{c}_{i,t,t'}}{\tilde{g}_t} & \le 24 \sqrt{\log\left( \frac{6(T+T^2)\Lambda K}{\delta} \right)} \sum_{t=0}^{T-2}\sum_{t'=t+1}^{T-1} \frac{\gamma^{t+t'}}{\tilde{g}_t \sqrt{i-1}} \\ & \le 48 \sqrt{\log\left( \frac{6(T+T^2)\Lambda K}{\delta} \right)} \sum_{t=0}^{T-2}\sum_{t'=t+1}^{T-1} \frac{\gamma^{t+t'}}{\sqrt{i}}
\end{align*}

\paragraph{Analysis of the initial phase}
We now analyze phase $i=1$. Here, the budget is allocated uniformly. Therefore, we have that:
\begin{align*}
    \sum_{t=0}^{T-1} f_t \left(\frac{1}{b/T} - \frac{1}{\tilde{g}_t} \right) \le  \sum_{t=0}^{T-1} \frac{f_t}{b/T} \le \sum_{t=0}^{T-1} f_t \le \left( \sum_{t=0}^{T-1} \gamma^t \right)^2  
\end{align*}

\paragraph{Plugging everything together}
Combining these results into Equation \eqref{eq:iterative-error-eq2}, we obtain that, with probability at least $1-\delta$: 
\begin{align}\label{eq:iterative-error-eq7}
    \frac{2}{K^2}\sum_{i=1}^K \left( 48 \log\left( \frac{6(T+T^2)\Lambda K}{\delta} \right) \left(\sum_{t=0}^{T-1} \gamma^t \right)^2 \right) \frac{1}{\sqrt{i}}
\end{align}
To conclude the proof, we observe that $\sum_{i=1}^n \frac{1}{\sqrt{i}} \le 2\sqrt{n} - 1$. Therefore:
\begin{align*}
    \frac{192}{K^{\frac{3}{2}}}  \log\left( \frac{6(T+T^2)\Lambda K}{\delta} \right) \left(\sum_{t=0}^{T-1} \gamma^t \right)^2 
\end{align*}
which is the desired result.
\end{proof}

We now continue by upper bounding another term of Equation \eqref{eq:regret}, that is:
\begin{align*}
\frac{2}{K} \sum_{t=0}^{T-1} \frac{f_t}{\tilde{x}^*_{t} + 1} - \sum_{t=0}^{T-1} \frac{f_t}{n_t^*}.
\end{align*}

\begin{lemma}[Exploration error]\label{lemma:exploration-lemma}Let $y \in \mathcal{Y}$, and define $q(y)$ as the smallest integer in $\left\{y+1, \dots, T-1 \right\}$ such that $\sum_{i=y}^{q(y)} f_i \ge 0$. Let $\bm{\tilde{x}}^*$ be the solution of the following optimization problem:
\begin{equation}
\begin{aligned} 
\min_{\bm{x}} \quad & \sum_{i=1}^{K} \frac{f_t}{x_t} \\
\textrm{s.t.} \quad &  \sum_{t=0}^{T-1} x_t = b - T \\
  & x_i \ge x_{i+1}, \quad \forall i \in \{1, \dots, K-1\}  \\
  & x_i \ge 0, \quad \forall i \in \{1, \dots, K\} \\
  & x_y = x_{y+1} = \dots = x_{q(y)}, \quad \forall y \in \mathcal{Y}.
\end{aligned}
\end{equation}
Then,
\begin{align*}
    \frac{2}{K} \sum_{t=0}^{T-1} \frac{f_t}{\tilde{x}^*_{t} + 1} - \sum_{t=0}^{T-1} \frac{f_t}{n_t^*} \le \frac{c+1}{c-1} \sum_{t=0}^{T-1} \frac{f_t}{x^*_t},
\end{align*}
where $c$ is such that $cT = b$, and $\bm{x}^*$ is the solution of the following optimization problem:
\begin{equation}
\begin{aligned} 
\min_{\bm{x}} \quad & \sum_{i=1}^{K} \frac{f_t}{x_t} \\
\textrm{s.t.} \quad &  \sum_{t=0}^{T-1} x_t = \Lambda \\
  & x_i \ge x_{i+1}, \quad \forall i \in \{1, \dots, K-1\}  \\
  & x_i \ge 0, \quad \forall i \in \{1, \dots, K\} \\
  & x_y = x_{y+1} = \dots = x_{q(y)}, \quad \forall y \in \mathcal{Y}.
\end{aligned}
\end{equation}
\end{lemma}
\begin{proof}
Consider the following optimization problem: 
\begin{equation}
\begin{aligned} 
\min_{\bm{x}} \quad & \sum_{i=1}^{K} \frac{f_t}{x_t} \\
\textrm{s.t.} \quad &  \sum_{t=0}^{T-1} x_t = K(b - T) \\
  & x_i \ge x_{i+1}, \quad \forall i \in \{1, \dots, K-1\}  \\
  & x_i \ge 0, \quad \forall i \in \{1, \dots, K\}, \\
  & x_y = x_{y+1} = \dots = x_{q(y)}, \quad \forall y \in \mathcal{Y}
\end{aligned}
\end{equation}
and let $\bm{\bar{x}}^*$ be its optimal solution. Then, due to Lemma \ref{lemma:restricted-budget-opt}, $K\bm{\tilde{x}}^* = \bm{\bar{x}}^*$. Therefore, we have that:
\begin{align*}
    \frac{2}{K} \sum_{t=0}^{T-1} \frac{f_t}{\tilde{x}^*_{t} + 1} - \sum_{t=0}^{T-1} \frac{f_t}{n_t^*} = 2 \sum_{t=0}^{T-1} \frac{f_t}{\bar{x}^*_{t} + K} - \sum_{t=0}^{T-1} \frac{f_t}{n_t^*} 
\end{align*}
Furthermore, due to the fact that $\bar{x}^*_{y} = \bar{x}^*_{y+1} = \dots = \bar{x}^*_{q(y)}$ for all $y \in \mathcal{Y}$, we have that:
\begin{align*}
    2 \sum_{t=0}^{T-1} \frac{f_t}{\bar{x}^*_{t} + K} - \sum_{t=0}^{T-1} \frac{f_t}{n_t^*} \le 2 \sum_{t=0}^{T-1} \frac{f_t}{\bar{x}^*_{t}} - \sum_{t=0}^{T-1} \frac{f_t}{n_t^*} 
\end{align*}

At this point, we proceed by lower bounding:
\begin{align*}
    \sum_{t=0}^{T-1} \frac{f_t}{n_t^*}.
\end{align*}
More specifically, consider the following optimization problem:
\begin{equation}
\begin{aligned} 
\min_{\bm{x}} \quad & \sum_{i=1}^{K} \frac{f_t}{x_t} \\
\textrm{s.t.} \quad &  \sum_{t=0}^{T-1} x_t = \Lambda \\
  & x_i \ge x_{i+1}, \quad \forall i \in \{1, \dots, K-1\}  \\
  & x_i \ge 0, \quad \forall i \in \{1, \dots, K\}, \\
  & x_y = x_{y+1} = \dots = x_{q(y)}, \quad \forall y \in \mathcal{Y},
\end{aligned}
\end{equation}
and let $\bm{x}^*$ be its optimal solution. Then, we have that:
\begin{align}\label{eq:approximation-error-eq1}
    \sum_{t=0}^{T-1} \frac{f_t}{n_t^*} \ge \sum_{t=0}^{T-1} \frac{f_t}{x^*_t}.
\end{align}
To prove Equation \eqref{eq:approximation-error-eq1}, it is sufficient to drop the integer constraints from the \eqref{sys:opt-prob}, then, due to Lemma \ref{lemma:opt-lemma-variance-funct}, we can impose the equality constraints on the resulting optimization problem, and finally, we enlarge the feasible region by setting the constraints $x_i \ge 0$. 

At this point, we have:
\begin{align}\label{eq:approximate-error-eq2}
    \frac{2}{K} \sum_{t=0}^{T-1} \frac{f_t}{\tilde{x}^*_{t} + 1} - \sum_{t=0}^{T-1} \frac{f_t}{n_t^*} \le 2 \sum_{t=0}^{T-1} \frac{f_t}{\bar{x}^*_t} - \sum_{t=0}^{T-1} \frac{f_t}{x^*_t}. 
\end{align}
By Lemma \ref{lemma:restricted-budget-opt}, we have that:
\begin{align*}
    \bar{x}^*_t = \frac{K(b-T)}{\Lambda} x^*_t = \frac{K(b-T)}{Kb} x^*_t = \frac{b-T}{b} x^*_t = \frac{cT-T)}{cT} x^*_t = \frac{c-1}{c} x^*_t 
\end{align*}
Plugging this result into Equation \eqref{eq:approximate-error-eq2}, we obtain:
\begin{align*}
    \frac{2}{K} \sum_{t=0}^{T-1} \frac{f_t}{\tilde{x}^*_{t} + 1} - \sum_{t=0}^{T-1} \frac{f_t}{n_t^*} & \le \left( \frac{2c}{c-1}  -1 \right) \sum_{t=0}^{T-1} \frac{f_t}{x^*_t} \\ & = \frac{c+1}{c-1} \sum_{t=0}^{T-1} \frac{f_t}{x^*_t}.
\end{align*}

\end{proof}

We now prove Theorem \ref{theo:iterative-regret}.

\iterativeregret*
\begin{proof}
From Lemma \ref{lemma:regret-decomposition}, we upper bound Equation \eqref{eq:regret} by:
\begin{equation}\label{eq:theo-regret-eq1}
\begin{aligned}
    \frac{2}{K^2} & \sum_{i=1}^{K} \sum_{t=0}^{T-1} f_t \left( \frac{1}{\bar{N}_{t,i}} - \frac{1}{\tilde{x}^*_t + 1} \right) + \frac{2}{K^2} \sum_{i=1}^{K} \sum_{t: f_t < 0} \frac{|f_t|}{\bar{N}_{t,i}}  - \sum_{t: f_t < 0} \frac{|f_t|}{\sum_{i=1}^{k} N_{t,i}} \\
    & + \frac{2}{K} \sum_{t=0}^{T-1} \frac{f_t}{\tilde{x}^*_{t} + 1} - \sum_{t=0}^{T-1} \frac{f_t}{n_t^*}.
\end{aligned}
\end{equation}
We notice that:
\begin{align*}
    \frac{2}{K^2} \sum_{i=1}^{K} \sum_{t: f_t < 0} \frac{|f_t|}{\bar{N}_{t,i}}  - \sum_{t: f_t < 0} \frac{|f_t|}{\sum_{i=1}^{k} N_{t,i}} \le \frac{2}{K} \sum_{t: f_t < 0} {f_t}
\end{align*}
Combining this inequality with Equation \eqref{eq:theo-regret-eq1}, we have that:
\begin{align}\label{eq:theo-regret-eq2}
    \frac{2}{K^2} \sum_{i=1}^{K} \sum_{t=0}^{T-1} f_t \left( \frac{1}{\bar{N}_{t,i}} - \frac{1}{\tilde{x}^*_t + 1} \right) + \frac{2}{K} \sum_{t: f_t < 0} {|f_t|}   + \frac{2}{K} \sum_{t=0}^{T-1} \frac{f_t}{\tilde{x}^*_{t} + 1} - \sum_{t=0}^{T-1} \frac{f_t}{n_t^*}.
\end{align}
Using Lemma \ref{lemma:iterative-sampling-error}, this can in turn be upper-bounded by:
\begin{align*}
192\left(\frac{b}{\Lambda}\right)^{\frac{3}{2}} \log\left( \frac{6(T+T^2)\Lambda K}{\delta} \right) \left(\sum_{t=0}^{T-1} \gamma^t \right)^2 + \frac{2}{K} \sum_{t: f_t < 0} {|f_t|}   + \frac{2}{K} \sum_{t=0}^{T-1} \frac{f_t}{\tilde{x}^*_{t} + 1} - \sum_{t=0}^{T-1} \frac{f_t}{n_t^*}.
\end{align*}
Applying Lemma \ref{lemma:exploration-lemma}, we further bound the previous equation with: 
\begin{align}\label{eq:theo-regret-eq3}
    192\left(\frac{b}{\Lambda}\right)^{\frac{3}{2}} \log\left( \frac{6(T+T^2)\Lambda K}{\delta} \right) \left(\sum_{t=0}^{T-1} \gamma^t \right)^2 + \frac{2}{K} \sum_{t: f_t < 0} {|f_t|} + \frac{c+1}{c-1}\sum_{t=0}^{T-1} \frac{f_t}{x^*_t},
\end{align}
where $x^*_t$ is the solution of the following optimization problem:
\begin{equation}
\begin{aligned} 
\min_{\bm{x}} \quad & \sum_{i=1}^{K} \frac{f_t}{x_t} \\
\textrm{s.t.} \quad &  \sum_{t=0}^{T-1} x_t = \Lambda \\
  & x_i \ge x_{i+1}, \quad \forall i \in \{1, \dots, K-1\}  \\
  & x_i \ge 0, \quad \forall i \in \{1, \dots, K\}, \\
  & x_y = x_{y+1} = \dots = x_{q(y)}, \quad \forall y \in \mathcal{Y},
\end{aligned}
\end{equation}
Furthermore, since:
\begin{align*}
    \sum_{t=0}^{T-1} \frac{f_t}{n^*_t} \ge \sum_{t=0}^{T_1} \frac{f_t}{x^*_t},
\end{align*}
Equation \ref{eq:theo-regret-eq3} reduces to:
\begin{align*}
    192\left(\frac{b}{\Lambda}\right)^{\frac{3}{2}} \log\left( \frac{6(T+T^2)\Lambda K}{\delta} \right) \left(\sum_{t=0}^{T-1} \gamma^t \right)^2 + \frac{2}{K} \sum_{t: f_t < 0} {|f_t|} + \frac{c+1}{c-1}\sum_{t=0}^{T-1} \frac{f_t}{n^*_t},
\end{align*}
The proof follows by noticing that:
\begin{align*}
    \frac{c+1}{c-1} = \frac{b+T}{b-T} \le 3,
\end{align*}
and, by isolating $\sum_{t=0}^{T-1} \frac{f_t}{N_t}$ in Equation \eqref{eq:regret}, and by noticing that $\Var[J_{\bm{n}^*}]$ scales with $\Lambda^{-1}$.
\end{proof}

\subsection{Additional Technical Details}\label{app-subsec:tech-details}
In this section, we provide additional techincal details that have been mentioned in the main text. More specifically, we provide (i) a formal description of the transformation between optimization problems and how we applied this technique in \rido, (ii) difficulties in deriving closed-form solutions for the optimization problems of interest, (iii) and theoretical evidence for the sub-optimality of non-adaptive methods whose variance cannot scale with the variance of the optimal DCS).

\subsubsection{Additional Details on solving the empirical optimization problem}
We begin with a more in-depth discussion of the transformation between optimization problems. Let $c_t \in \mathbb{R}$ for each $t \in \{0, \dots, T-1 \}$, and define $\mathcal{Y} = \left\{i \in \left\{0, \dots, T-1 \right\} : c_i < 0 \right\}$. Let $y \in \mathcal{Y}$, and define $q(y)$ as the smallest integer in $\left\{y+1, \dots, T-1 \right\}$ such that $\sum_{i=y}^{q(y)} c_i \ge 0$. Due to Lemma \ref{lemma:var-funct-property} we know that, if $\left(c_0, \dots, c_{T-1}\right) = \left(f_0, \dots, f_{T-1}\right)$, then $q(y)$ is always well-defined. At this point, consider the continuous relaxation of the original optimization problem, namely:
\begin{equation}\label{sys:opt-prob-relaxed}
\begin{aligned} 
\min_{\bm{n}} \quad & \sum_{t=0}^{T-1} \frac{1}{n_t} \left( \gamma^{2t} \Var(R_t) + 2 \sum_{t'=t+1}^{T-1} {\gamma^{t+t'}} \Cov(R_t, R_{t'}) \right)  \\
\textrm{s.t.} \quad &  \sum_{t=0}^{T-1} n_t = \Lambda \\
  & n_t \ge n_{t+1}, \quad \forall t \in \{0, \dots, T-2\}  \\
  & n_t \ge 1 , \quad \forall t \in \{0, \dots, T-1\}.
\end{aligned}
\end{equation}
Due to Lemma \ref{lemma:var-funct-property} and Lemma \ref{lemma:opt-lemma-variance-funct}, we know that the following optimization problem:
\begin{equation}\label{sys:opt-prob-relaxed}
\begin{aligned} 
\min_{\bm{n}} \quad & \sum_{t=0}^{T-1} \frac{1}{n_t} \left( \gamma^{2t} \Var(R_t) + 2 \sum_{t'=t+1}^{T-1} {\gamma^{t+t'}} \Cov(R_t, R_{t'}) \right)  \\
\textrm{s.t.} \quad &  \sum_{t=0}^{T-1} n_t = \Lambda \\
  & n_t \ge n_{t+1}, \quad \forall t \in \{0, \dots, T-2\}  \\
  & n_t \ge 1 , \quad \forall t \in \{0, \dots, T-1\} \\
  & n_y = n_{y+1} = \dots = n_{q(y)} , \quad \forall y \in \mathcal{Y},
\end{aligned}
\end{equation}
has the same optimal solution of \eqref{sys:opt-prob-relaxed}. At this point, to define the transformed problem it is sufficient to introduce additional variables $y_i$ for any contiguous timesteps where $n_i = n_{i+1} = \dots = n_{i+k}$ holds for some integers $i, k$.\footnote{More precisely, we notice that $n_i = n_{i+1} = \dots = n_{i+k}$ might involve multiple constraints in the formulation of \eqref{sys:opt-prob-relaxed}. In this sense, we need to refer to the largest intervals in which these constraints are enforced, otherwise we might introduce multiple variables that refer to the same original optimization variable.} The optimization variables $n_i, n_{i+1}, \dots, n_{i+k+1}$ will be substituted with $y_i$. The objective function will be modified accordingly, namely:
\begin{align*}
\frac{f_i}{n_i} + \dots + \frac{f_{i+k}}{n_{i+k}},
\end{align*} 
is replaced with:
\begin{align}\label{eq:numerator}
\frac{f_y + \dots + f_{q(y)}}{y_i}.
\end{align}
%At this point, due to the property that the numerator in Equation \eqref{eq:numerator} is greater of equal than $0$ by construction.\am{Questa frase sembra incompleta} 
Consequently, any numerator in the resulting objective function of the transformed problem will be greater or equal than $0$. It is easy to verify that, in this case, the resulting objective function is convex in the considered optimization domain. Finally, as a last remark, we notice that the constraint $\sum_{t=0}^{T-1} n_t = \Lambda$ needs to be modified. More specifically, if $y_i$ substitutes $l_i$ variables, then its contribution within the budget constraint summation will be given by $y_i l_i$.

As discussed in Section \ref{sec:emp-problem}, in \rido we adopt a procedure that is inspired by the aforementioned theoretical properties of the continuous relaxation of the optimization problem \eqref{sys:opt-prob}. Nevertheless, it has to be noticed that a modification needs to be taken into account when replacing exact quantities (i.e., $f_t$) with their estimation and exploration bonuses (which, in the following, we refer to as $\hat{f}_t$ for brevity). More specifically, in general, contrary to what highlighted in Lemma \ref{lemma:var-funct-property} for the original objective function, when using $\hat{f}_t$ it might happen that $q(y)$ is not well-defined for every possible $y$. Indeed, due to the noise that is present in the estimation process, there might exists $\bar{t}$ such that $\hat{f}_{\bar{t}} < 0$ and $\sum_{t=\bar{t}}^{t'} \hat{f}_t < 0$ for all $t' > \bar{t}$. Whenever this condition is verified, we adopt the following heuristic to make the computation tractable. If $\bar{t} = 0$, then we just set the DCS of the current mini-batch to the uniform-in-the-horizon one. When $\bar{t} \ne 0$, instead, we group together $n_{\bar{t}}, \dots, n_{T-1}$ and we introduce a new variable $y$ that will divide, in the objective function, $\hat{f}_{\bar{t} - 1}$. 
As a final remark, however, we notice that these modifications do not impact on the theoretical properties of \rido. Indeed, Lemma \ref{lemma:high-prob-convex-problem}, shows that, with probability at least $1-\delta$, the aforementioned ill-conditions do not happen. As a consequence, we can study the high-probability behavior of \rido assuming access to the solution of the transformed optimization problem discussed at the beginning of this section (that preserves the optimal solution of the continuous relaxation of \eqref{sys:emp-opt-prob}). 

\subsubsection{On closed-form solutions}
We now continue by discussing the closed-form solutions of the optimization problems of interests. First of all, optimization problems \eqref{sys:opt-prob} and \eqref{sys:emp-opt-prob} are integer and non-linear problems. Even neglecting the non-linear dependency on $n$, we remark that solving integer and linear problem is NP-hard. At this point, one might resort to study their continuous relaxations. In the following, we focus on the continuous relaxation of \eqref{sys:opt-prob} (indeed, as noticed at the end of the previous section, the continuous relaxation of \eqref{sys:emp-opt-prob} requires additional effort). As mentioned above, whenever $f_t < 0$ holds for some $t \in \{0, \dots, T-1\}$, the continuous relaxation of $f_t$ is non-convex. Nevertheless, from Lemma \ref{lemma:opt-lemma-variance-funct}, we know that we can always derive an equivalent convex problem (where the numerator in the objective function is always greater or equal than $0$) that preserves the optimal solution. For this reason, we now report the KKT conditions under the assumption that $f_t \ge 0$ holds for all $t \in \{0, \dots, T-1\}$.\footnote{Under the assumption that $f_t \ge 0$ holds, the problem is convex, and the KKT conditions provides necessary and sufficient conditions for optimality.}

\begin{equation}\label{sys:kkt}
	\begin{cases}
			-\frac{f_t}{n_t^2} + \eta - \xi_t - \mu_{t} \bm{1}\left\{t < T-1 \right\} + \mu_{t-1} \bm{1} \left\{t > 0 \right\} = 0 & \text{ } \forall t \in  \{0, \dots, T-1 \}\\
			\xi_t (1-{n}_t) = 0 & \text{ } \forall t \in \{0, \dots, T-1\} \\
			\mu_t (n_{t+1} - n_t) = 0 & \text{ } \forall t \in \{0, \dots, T-2 \} \\
			\eta (\sum_{t=0}^{T-1} {n}_t - \Lambda) = 0 \\
			\sum_{t=0}^{T-1} {n}_t - \Lambda = 0 \\
			\mu_t \ge 0 & \text{ } \forall t \in  \{0, \dots, T-2 \} \\
			\xi_t \ge 0 & \text{ } \forall t \in \{0, \dots, T-1 \} 
		\end{cases}.
\end{equation}

At this point, we notice that a similar problem has been solved in \citet{poiani2022truncating} for deriving a closed-form solutions that minimizes confidence intervals around the return estimator.  In that situation, however, the constraints $n_t \ge n_{t+1}$ were not present since they were automatically satisfied by any optimal solution (and, consequently, they were removed from the optimization problem of interest). The main challenge in our setting is, indeed, the presence of $\mu_t (n_{t+1} - n_t)$, together with the terms related to $\mu_t$ in the first Equation of \eqref{sys:kkt}. These additional components within \eqref{sys:kkt} prevented us to derive a closed-form solutions of the continuous relaxation \eqref{sys:opt-prob} (and, \eqref{sys:emp-opt-prob}).

\subsubsection{Theoretical Sub-optimality of pre-determined schedules}
Finally, we conclude by providing theoretical evidence on the reasons why claims similar to the one of Theorem~\ref{theo:iterative-regret} does not hold for pre-determined schedules (i.e, the uniform-in-the-horizon one and the robust DCS of \citet{poiani2022truncating}). In other words, we show that it is not possible for these methods to obtain a variance expression which is proportional, up to constant multiplicative factors, to the one of the optimal baseline.

\begin{proposition}[Sub-optimality of the Uniform Strategy]\label{prop:uni}
Let $f_0 \ne 0$ and $f_i = 0$ for all $i \ge 1$. Let $T > 2$. Let $\bm{n}_u = \left( \frac{T}{\Lambda}, \dots, \frac{T}{\Lambda} \right)$. For any value of budget $\Lambda$, it does not exist a universal constant $c > 0$ for which the following holds:
\begin{align}\label{prop-uni:claim}
\Var_{\bm{n_u}}\left[ \hat{J}_{\bm{n_u}}(\pi) \right] \le c \Var_{\bm{n^*}}\left[ \hat{J}_{\bm{n^*}}(\pi) \right].
\end{align}
\end{proposition}
\begin{proof}
Under the assumption that $f_0 \ne 0$ and $f_i = 0$ for all $i \ge 1$, we have that:
\begin{align}\label{prop-uni:eq1}
\Var_{\bm{n_u}}\left[ \hat{J}_{\bm{n_u}}(\pi) \right] = \frac{T}{\Lambda} f_0,
\end{align}
and, from Theorem \ref{theo:dcs-variance}:
\begin{align}\label{prop-uni:eq2}
\Var_{\bm{n^*}}\left[ \hat{J}_{\bm{n^*}}(\pi) \right] = \frac{1}{\Lambda - (T-1)} f_0.
\end{align}
Furthermore, if $T > 2$, the variance of the optimal DCS can be upper bounded by:
\begin{align}\label{prop-uni:eq3}
\Var_{\bm{n^*}}\left[ \hat{J}_{\bm{n^*}}(\pi) \right] = \frac{1}{\Lambda - (T-1)} f_0 \le \frac{2}{\Lambda} f_0.
\end{align}
At this point, proceed by contradiction and suppose that Equation \eqref{prop-uni:claim} holds. Then, it follows that the following equation should holds as well for some universal constant $c$:
\begin{align}\label{prop-uni:eq4}
\frac{T}{\Lambda} f_0 \le c \frac{2}{\Lambda} f_0 .
\end{align}
Equation \eqref{prop-uni:eq4} reduces to:
\begin{align}
c \ge \frac{T}{2},
\end{align}
which contradicts the claim, thus concluding the proof.
\end{proof}

\begin{proposition}[Sub-optimality of the Robust Strategy of \citet{poiani2022truncating}]\label{prop:robust}
Let $f_{0} \ne 0$ and $f_i = 0$ for all $i \ge 1$. Let $\bm{\tilde{n}}$ be the robust DCS of \citet{poiani2022truncating}. Let $T > 2$ and $d_t = \frac{\gamma^t ( \gamma^t + \gamma^{t+1} - 2\gamma^T)}{1-\gamma}$ and suppose that $\Lambda \ge \Lambda_0 \coloneqq \frac{\sum_{t=0}^{T-1} \sqrt{d_t}}{\sqrt{d_{T-1}}}$. For any value of budget $\Lambda \ge \Lambda_0$, it does not exist a universal constant $c > 0$ for which the following holds
\begin{align}\label{prop-rob:claim}
\Var_{\bm{\tilde{n}}}\left[ \hat{J}_{\bm{\tilde{n}}}(\pi) \right] \le c \Var_{\bm{n^*}}\left[ \hat{J}_{\bm{n^*}}(\pi) \right].
\end{align}
\end{proposition}
\begin{proof}
Under the assumption that $f_0 \ne 0$, $f_i = 0$ for all $i \ge 1$, and $\Lambda \ge \Lambda_0$ we have that:\footnote{Notice that the requirement $\Lambda \ge \Lambda_0$ provides a simple closed-form expression for the robust DCS of \citet{poiani2022truncating}. The reader can refer to Theorem $3.3$ and Appendix B of \citet{poiani2022truncating}.}
\begin{align}\label{prop-rob:eq1}
\Var_{\bm{\tilde{n}}}\left[ \hat{J}_{\bm{\tilde{n}}}(\pi) \right] \ge \frac{f_0}{2\Lambda} \frac{ \sum_{t=0}^{T-1} \sqrt{d_t} }{\sqrt{d_0}}.
\end{align}
From Theorem \ref{theo:dcs-variance}:
\begin{align}\label{prop-rob:eq2}
\Var_{\bm{n^*}}\left[ \hat{J}_{\bm{n^*}}(\pi) \right] = \frac{1}{\Lambda - (T-1)} f_0.
\end{align}
Furthermore, if $T > 2$, the variance of the optimal DCS can be upper bounded by:
\begin{align}\label{prop-rob:eq3}
\Var_{\bm{n^*}}\left[ \hat{J}_{\bm{n^*}}(\pi) \right] = \frac{1}{\Lambda - (T-1)} f_0 \le \frac{2}{\Lambda} f_0.
\end{align}
At this point, proceed by contradiction and suppose that Equation \eqref{prop-rob:claim} holds. Then, it follows that the following  equation should hold as well for some universal constant $c$:
\begin{align}\label{prop-rob:eq4}
\frac{f_0}{\Lambda} \frac{ \sum_{t=0}^{T-1} \sqrt{d_t} }{\sqrt{d_0}} \le c \frac{2}{\Lambda} f_0 .
\end{align}
Equation \eqref{prop-rob:eq4} can be rewritten as:
\begin{align}\label{prop-rob:eq5}
c \ge \frac{1}{2} \frac{ \sum_{t=0}^{T-1} \sqrt{d_t} }{\sqrt{d_0}}.
\end{align}
However, if Equation \eqref{prop-rob:eq5} holds, then, also the following holds:
\begin{align}\label{prop-rob:eq6}
c \ge \frac{1}{4} \sum_{t=0}^{T-1} \sqrt{\gamma^t \left( \gamma^t + \gamma^{t+1} - \gamma^{2T} \right)},
\end{align}
which, however, contradicts the claim \footnote{Indeed, it is sufficient to take $T \rightarrow +\infty$, and $\gamma \rightarrow 1$, to show that Equation \eqref{prop-rob:eq6} tends to $+\infty$.}, thus concluding the proof.
\end{proof}

\section{Experiment Details and Additional Results}\label{app:experiments}
In this section, we provide further details on our experimental settings and additional results. Section \ref{app:env-details} contains descriptions on the environments, Section \ref{app:hyper-params} contains details regarding hyper-parameters, and Section \ref{app:add-results} contains additional results.

\subsection{Environment Details}\label{app:env-details}
In this section, we provide additional details on the environments that we used in our experiments. 

\paragraph{Ablation Domains} In Setion \ref{app:add-results}, the reader can find results and ablations that involve the scenarios described as examples in Section \ref{sec:intro}, namely Examples \ref{exe:1} and \ref{exe:2}. We now provide a precise description of these domains. We start with Example \ref{exe:1}, where the reward is gathered only at the end of the estimation horizon $T$. The state space is described by a $1$-dimensional vector that contains only the interaction timestep $t$; the action space is a discrete set $\left\{ 0, 1\right\}$. The agent receives reward $0$ in the first $9$ timesteps. In the last step, instead, it receives $r \sim \mathcal{N} \left(3, 10 \right)$ for action $0$, and $r \sim \mathcal{N} \left(2, 10 \right)$ for action $1$. Concerning Example \ref{exe:2}, instead, the setup is identical to the one of Example \ref{exe:1}, with the only different that the non-zero reward is receives in the first interaction step. The policy that we evaluate is the uniform random.

\paragraph{Continuous Navigation} Here, we describe in more details the 2D continuous navigation environment that we used in our experiments. The state space $\mathcal{S}$ is $2$-dimensional vector $\bm{s} = (s_0, s_1) \in \mathbb{R}^{2}$ such that $s_i \in [0, 92]$ for all $i$. Similarly, the action space $\mathcal{A}$ is a $2$-dimensional vector $\bm{a} = (a_0, a_1)$ such that $a_i \in [-1, 1]$ for all $i$. When the agent takes action $\bm{a}$ in state $\bm{s}$, it transitions to a new state $\bm{s'}$ such that:
\begin{align*}
s'_0 = \max\left\{ 0, \min \left\{ s_0 + q_0, 92 \right\}\right\}, \quad \quad \quad s'_1 = \max\left\{ 0, \min \left\{ s_1 + q_1, 92 \right\}\right\},
\end{align*}
where $q_0 \sim \mathcal{N}\left(a_0, 0.1 \right)$, $q_1 \sim \mathcal{N}\left(a_1, 0.1 \right)$, and the max-min operations simply guarantees that the resulting state lies within the desired state space $\mathcal{S}$. 
The agent receives rewards egual to $0$ at every time step, except when the resulting state $\bm{s'}$ falls within a goal region. More specifically, the goal is defined as a $2$-dimensional vector 
$\bm{g} = (91, 91)$. Whenever $||\bm{s'} - \bm{g}||_2 \le 1$ the reward received by the agent is sampled from the following Gaussian distribution: $\mathcal{N}\left(1, 1\right)$. The agent starts in a random position that is sampled from a uniform distribution in the area $[0, 5] \times [0, 5]$. The agent policy that we evaluate in our experiments is an hand-coded expert policy that minimizes the distance between the agent's position and the center of the goal area. More specifically, given the agent position $\bm{s}$, $\bm{a}$ is computed in the following way.
\begin{align*}
a_0 = \max \left\{ -1, \min \left\{ g_0 - s_0, 1 \right\} \right\}, \quad \quad \quad a_1 = \max \left\{ -1, \min \left\{ g_1 - s_1, 1 \right\} \right\},
\end{align*}
where the max-min operation guarantees that $\bm{a}$ belongs to $\mathcal{A}$.

\paragraph{LQG} Concerning the LQG, we consider the following $1$-dimensional case (i.e., the dimension of the state and action spaces is $1$). The initial state is drawn from a uniform distribution in $[-80, +80]$. Upon taking action $a \in \mathcal{A}$, the agent transitions to a new state $s' = s + (a + \xi) + \eta$, where $\eta \sim \mathcal{N}\left(0, 0.1 \right)$ models the noise in the system, and $\xi \sim \mathcal{N} \left(0, 0.1 \right)$ denotes the controller's noise. The reward for taking action $a$ in state $s$ is computed as $s^2 + (a + \xi)^2$. The policy that we evaluate is the optimal one and it is computed by solving the Riccati equations.

\paragraph{MuJoCo suite} In the main text, we presented results on the Ant environment of the MuJoCo suite. In the appendix, we present additional experiments on the HalfCheetah and Swimmer domains \citep{todorov2012mujoco}. In all cases, we adopted trained deep RL agents made publicly available by \citet{rlzoo3} (MIT License).

\subsection{Hyper-parameters}\label{app:hyper-params}
Table~\ref{tab:hp} reports the hyper-parameters that we used in our experiments. In each domain, we set $\beta=1.0$. As a consequence, the confidence intervals do not appear in the empirical optimization problem which is in practice solved by \rido. In practice, this choice gives good empirical performance, and limits the hyper-parameter tuning of our algorithm. 

\begin{table}
  \caption{Hyper-Parameters}
  \label{tab:hp}
  \centering
  \begin{tabular}{lll}
    \toprule
    Environment     						&  $\beta$		 & Mini-batch size \\
    \midrule
    Pendulum 							& $1.0$  		  & $500$   \\
    LQG     								& $1.0$		      & $500$   \\
    $2$D Continuous Navigation     		& $1.0$       	  & $1000$  \\
    Ant									& $1.0$       	  & $3000$  \\
    HalfCheetah							& $1.0$ 		      & $3000$ 	\\
    Swimmer								& $1.0$ 		      & $1000$ 	\\   
    Reacher								& $1.0$ 		      & $450$ 	\\   
    \bottomrule
  \end{tabular}
\end{table}

\subsection{Additional Results}\label{app:add-results}

\subsubsection{Ablations}
In this section, we present ablations on \rido on the two environments (described in Section \ref{app:env-details}) that models Examples \ref{exe:1} and \ref{exe:2}. More specifically, we conduct the following two ablations to understand the behavior of \rido according to changes in its hyper-parameters, i.e., the robustness level $\beta$ and the mini-batch size. To properly assess the effect of these designer's choices, we report and discuss both the average MSE and the resulting DCSs. We test our method using $\gamma=1$, but similar results can be obtained varying the discount factor.

\begin{figure}[t]
  \centering
  \includegraphics[width=13.5cm]{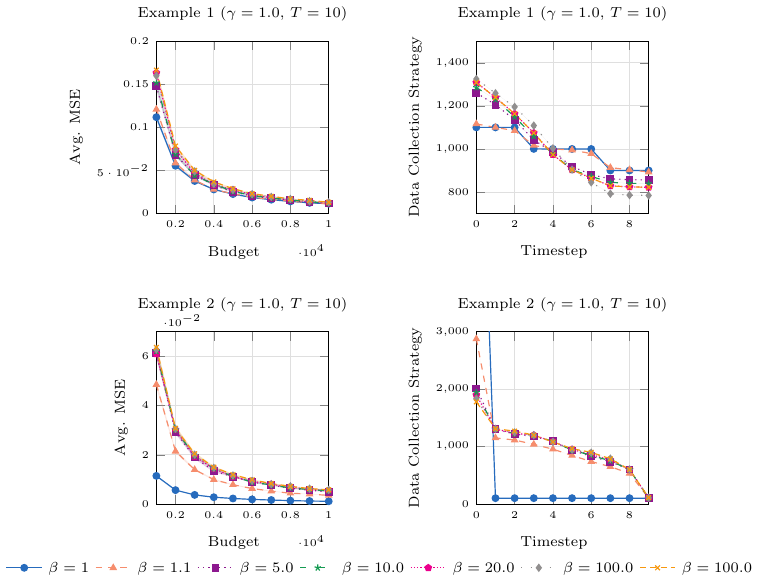} 
  \caption{Ablations on different values of $\beta$ on Examples \ref{exe:1} (\emph{top}) and \ref{exe:2} (\emph{bottom}). Empirical MSE (mean and $95$\% confidence intervals over $100$ runs) (\emph{left}). DCS visualiaztion (mean and $95$\% confidence intervals over $100$ runs) using $\Lambda = 10000$ (\emph{right}).}
  \label{fig:beta-abl}
\end{figure}

\begin{figure}[t]
  \centering
  \includegraphics[width=13.5cm]{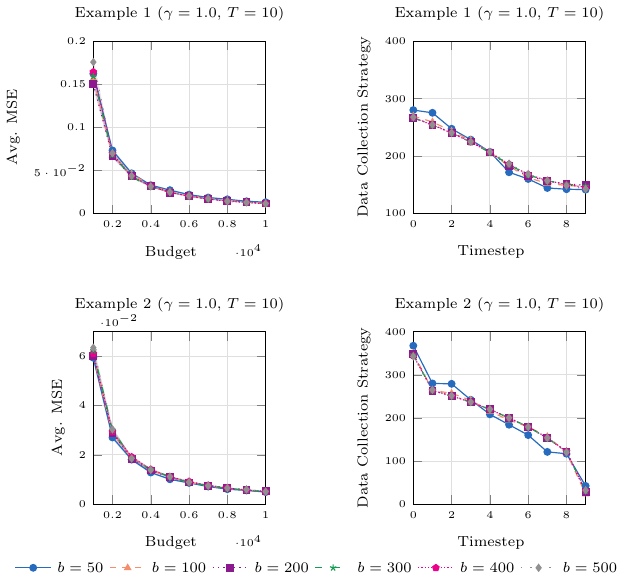} 
  \caption{Ablations on different mini-batch sizes on Examples \ref{exe:1} (\emph{top}) and \ref{exe:2} (\emph{bottom}). Empirical MSE (mean and $95$\% confidence intervals over $100$ runs) (\emph{left}).  DCS visualiaztion (mean and $95$\% confidence intervals over $100$ runs) using $\Lambda = 1000$ (\emph{right}).}
  \label{fig:batch-abl}
\end{figure}

\paragraph{Ablations on $\beta$} We begin by performing an ablation on the robustness parameter $\beta$. More specifically, we analyze the behavior of \rido for the following values of $\beta$: $\left\{1, 1.1, 5.0, 10.0, 20.0, 100.0, 1000.0 \right\}$ (the value of the mini-batch size here is fixed to $b=100$). Figure~\ref{fig:beta-abl} reports the results. Let us first focus on Example \ref{exe:1} (i.e., the top row). In this case, the reward is gathered at the end of the estimation horizon. As we can see, increasing the value of $\beta$, leads to a larger amount of data spent in the first interaction steps (i.e., top-right in Figure~\ref{fig:beta-abl}). Indeed, when higher values of $\beta$ are used, the cumulative sum of exploration bonuses in the early steps is larger w.r.t. the late ones. For this reason, \rido spends a larger portion of its budget to decrease these exploration bonuses. As a consequence, given that the reward process of the underlying environment, this results in a higher empirical MSE (i.e., top-left in Figure~\ref{fig:beta-abl}). Furthermore, given that the reward is $0$ everywhere except at $t=T-1$, even using the smallest value of $\beta$ (i.e., $\beta=1$) allows the algorithm to quickly adapt its DCS toward the most relevant timestep (i.e., $t=T-1$). Similar comments hold for Example~\ref{exe:2} as well (i.e., bottom row in Figure~\ref{fig:beta-abl}). Finally, we notice that the behavior changes almost unsignificantly for values of $\beta$ larger than $5.0$.

\paragraph{Ablations on mini-batch size $b$} We now continue by presenting an ablation on the batch size. More specifically, we analyze the behavior of \rido for the following values of $b$: $\left\{50, 100, 200, 300, 400, 500 \right\}$ (the value of $\beta$ here is fixed to $5.0$). Figure~\ref{fig:batch-abl} reports the results. First of all, as we can notice, in both Examples \ref{exe:1} and \ref{exe:2} the mini-batch size impacts the performance in a less significant way w.r.t. to the value of $\beta$ (compare the left column of Figure \ref{fig:batch-abl} and \ref{fig:beta-abl}). Secondly, let us focus on the the top-row (i.e., Example \ref{exe:1}, where the reward is gathered at the end of the episode). For the smallest value of $\Lambda$ of Figure~\ref{fig:batch-abl} (i.e., $\Lambda=1000$, that is the only for which there is some difference in performance), we notice that the best configuration is not $b=50$ (i.e., the smallest batch-size among the presented ones). This is confirmed also by its corresponding DCS, which is not the one that allocates the highest number of data at $T-1$. We conjecture that the reason behind this phenonema are numerical instabilities that might arise while solving the empirical problem with the use of convex solvers.\footnote{We notice that even small imprecisions can result in DCSs that differ by $1$ when converting the continuous relaxation to its integer version. For smaller values of $b$, this behavior might happen multiple times w.r.t. larger values of $b$.} Concerning Example~\ref{exe:2} (where the reward is gathered only at $t=0$), we notice that smaller values of $b$ performs better (this is confirmed by the corresponding DCS, that allocates more data to the first interaction step). In this case, the aformentioned problem is not present. We conjecture that the reason is that, even in the case of numerical instabilities, errors that arise from converting the continuous DCS to its integer version provably minimizes the MSE, since the remaining budget is allocated uniformly starting from $t=0$ (i.e., the most relevant timestep from the point of view of the estimation quality). Finally, we notice that, whatever value of $b$ we use, the behavior of \rido is stable under reasonable variations of the mini-batch size.

\subsubsection{DCS Visualizations for Figure~\ref{fig:mainresults}}
In this section, we present visualizations of the DCSs for the experiments presented in Figure~\ref{fig:mainresults}. Figure~\ref{fig:dcsvisuno} reports our results (mean and $95$\% confidence intervals over $100$ runs). The resulting visualizations reinforce the adaptivity of \rido. Indeed, depending on the domain, the behavior of \rido changes significantly, resulting in behaviors that are similar to the uniform strategy (i.e., Navigation), to the robust strategy (i.e., Ant), or significantly different from both pre-computed schedules (i.e., Pendulum and LQG).

\begin{figure}[t]
  \centering
  \includegraphics[width=13.5cm]{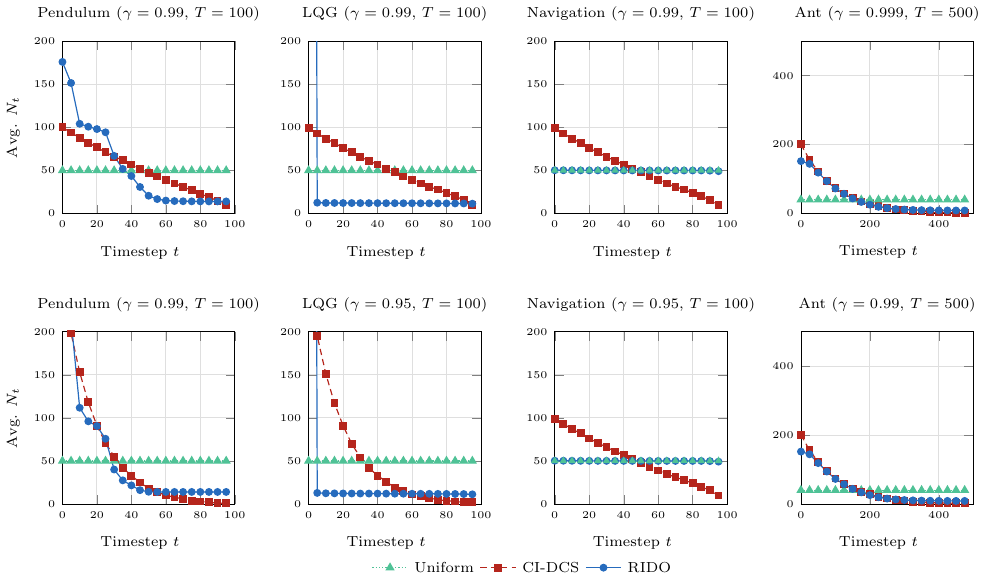} 
  \caption{DCS visualiaztion for Pendulum, LQG, Navigation and Ant (mean and $95$\% confidence intervals over $100$ runs). The $x$ axis reports the timestep $t$, while the $y$ axis $n_t$. For Pendulum, LQG and Navigation, we consider $\Lambda=5000$, while for the Ant $\Lambda=20000$.} 
  \label{fig:dcsvisuno}
\end{figure}

\subsubsection{Results on Additional Environments}
In this section, we present results on additional MuJoCo environments, namely Swimmer, Reacher and Half-Cheetah. Figure~\ref{fig:additional} reports our results.

\begin{figure}
  \centering
  \includegraphics[width=13.5cm]{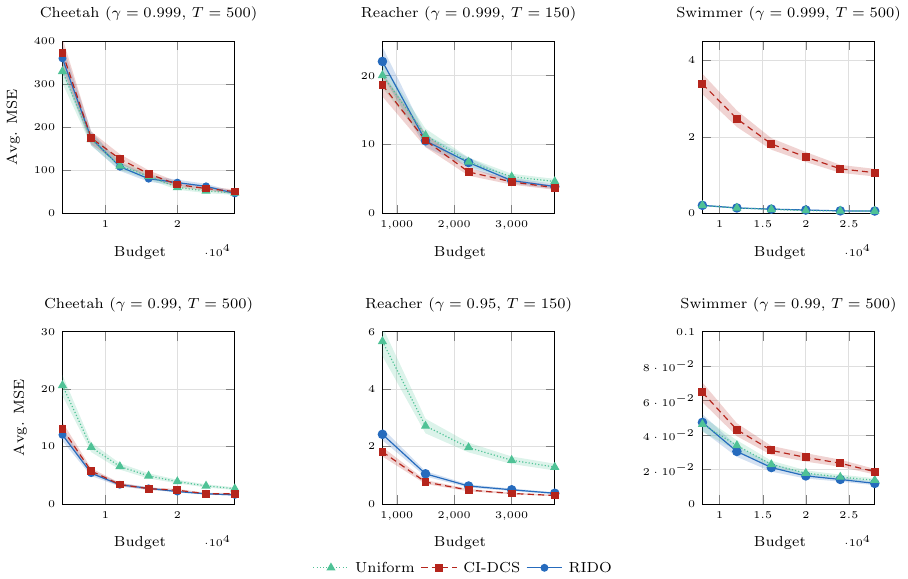} 
  \caption{Empirical MSE (mean and $95$\% confidence intervals over $100$ runs) on the Swimmer, Reacher and HalfCheetah for the considered baselines.}
  \label{fig:additional}
\end{figure}

\subsection{Additional Details on the Running Time}
In this section, we provide additional details on the running time of the algorithms.

More specifically, we have run all algorithms on Ant and HalfCheetah domain with $\Lambda = 4000$ and $T = 500$.  For our method, we have used a batch size of $1000$. For the Ant, our method took roughly $27$ seconds, while for uniform and \citet{poiani2022truncating}, the run took roughly $20$ seconds. In the HalfCheetah, instead, RIDO took $23s$, while the baselines $15s$. As soon as we increase $\Lambda = 8000$ (keeping the batch size $1000$), we obtain $47s$ for RIDO, and $29s$ for the baselines (Ant environment). For HalfCheetah, instead, $37s$ for RIDO, and $20s$ for the baselines. At this point, increasing the batch size to $2000$, RIDO obtains $41s$ in the Ant, and $31s$ in the HalfCheetah.  

Overall, RIDO requires some additional computational overheads, nevertheless, the running time is still comparable. Furthermore, we also notice that the code of our algorithm has not been specifically optimized for time efficiency. Finally, we also notice that in our experiments, we rely on open-source solvers. Relying on commercial solvers might increase the computational efficiency of our method.

\end{document}